\documentclass[11pt]{article}
\usepackage{graphicx} 
\usepackage{float}    
\usepackage{verbatim} 
\usepackage{amsmath}  
\usepackage{amssymb}  
\usepackage{subfig}   
\usepackage[colorlinks=true,citecolor=blue,linkcolor=blue,urlcolor=blue]{hyperref}
\usepackage{fullpage}
\usepackage{amsthm,mathtools}
\usepackage{enumerate}
\usepackage{paralist}
\usepackage{xspace}
\usepackage{caption}
\usepackage{bbm}
\usepackage{url}
\usepackage{mathrsfs}
\usepackage{algorithm}
\usepackage{algorithmic}
\usepackage{color}
\usepackage{microtype}
\usepackage{natbib}
\usepackage[section]{placeins}
\usepackage{subfig}   
\usepackage[colorlinks=true,citecolor=blue,linkcolor=blue,urlcolor=blue]{hyperref}
\usepackage{amsthm,mathtools}
\usepackage{enumerate}
\usepackage{paralist}
\usepackage{xspace}
\usepackage{caption}
\usepackage{bbm}
\usepackage{url}
\usepackage{mathrsfs}
\usepackage{algorithm}
\usepackage{algorithmic}
\usepackage{color}
\usepackage{microtype}
\usepackage{graphicx,wrapfig,lipsum, booktabs}

\makeatother

\newcommand{\R}{\mathbb{R}}

\newcommand{\E}{\mathbb{E}}

\newcommand{\N}{\mathcal{N}}
\newcommand{\rank}{\mathrm{rank}}
\newcommand{\sign}{\mathrm{Sign}}
\newcommand{\norm}[1]{\left\lVert#1\right\rVert}
\newcommand{\tr}[1]{\mathrm{Tr}\left(#1\right)}

\newcommand{\inner}[2]{\left\langle #1, #2 \right\rangle}

\DeclareMathOperator{\Cov}{Cov}
\DeclareMathOperator{\Var}{Var}

\newtheorem{theorem}{Theorem}
\newtheorem{definition}{Definition}
\newtheorem{lemma}{Lemma}

\newtheorem{remark}{Remark}
\newtheorem{corollary}{Corollary}
\newtheorem{proposition}{Proposition}

\newtheorem{assumption}{Assumption}
\newtheorem{claim}{Claim}
\newtheorem{fact}{Fact}

\title{\bf Sign-RIP: A Robust Restricted Isometry Property for Low-rank  Matrix Recovery}

\author{Jianhao Ma and Salar Fattahi\vspace{5mm}\\
	Industrial and Operations Engineering, University of Michigan
}

\begin{document}

\maketitle

\begin{abstract}
    Restricted isometry property (RIP), essentially stating that the linear measurements are approximately norm-preserving, plays a crucial role in studying low-rank matrix recovery problem.
    However, RIP {fails} in the robust setting, when a subset of the measurements are grossly corrupted with noise. 
    In this work, we propose a robust restricted isometry property, called \textit{Sign-RIP}, and show its broad applications in robust low-rank matrix recovery. In particular, we show that Sign-RIP can guarantee the uniform convergence of the subdifferentials of the robust matrix recovery with nonsmooth loss function, even at the presence of arbitrarily dense and arbitrarily large outliers. Based on Sign-RIP, we characterize the location of the critical points in the robust rank-$1$ matrix recovery, and prove that they are either close to the true solution, or have small norm. Moreover, in the over-parameterized regime, where the rank of the true solution is over-estimated, we show that subgradient method converges to the true solution at a (nearly) dimension-free rate. Finally, we show that sign-RIP enjoys almost the same complexity as its classical counterparts, but provides significantly better robustness against noise.
\end{abstract}

\section{Introduction}
Inspired by the surprising success of simple local-search algorithms in nonconvex optimization arising in modern machine learning tasks, a recent body of work focuses on studying the local and global optimization landscape of these problems. A prototypical class of such problems is {low-rank matrix recovery}, where the goal is to recover a low-rank matrix from a limited number of linear and noisy measurements. Low-rank matrix recovery is the cornerstone for many modern machine learning problems, including motion detection~\cite{bouwmans2014robust}, face recognition~\cite{luan2014extracting}, recommender systems~\cite{luo2014efficient}, and system identification~\cite{liu2010interior, chandrasekaran2011rank}. 

Despite the inherent difficulty of low-rank matrix recovery in its worst case---a fact noted as early as 1995~\cite{natarajan1995sparse}---it is known that {convex relaxation} methods can correctly recover the low-rank matrix under the so-called \textit{restricted isometry property} (RIP)~\cite{recht2010guaranteed, candes2008restricted, zhang2013restricted}, but suffer from high computational cost. One of the  breakthrough results in this line of research was presented in a 2016 NeurIPS paper~\cite{ bhojanapalli2016global}, showing that, for smooth low-rank matrix recovery, simple saddle-escaping algorithms, such as perturbed gradient descent (GD)~\cite{ge2015escaping, jin2017escape}, provably converge to the true low-rank solution. The main intuition behind this result is a simple, yet striking one: {\it under the same RIP condition, the nonconvex formulation of smooth low-rank matrix recovery problem is devoid of spurious local minima.} This result lead to a flurry of follow-up papers characterizing the landscape of other variants of low-rank matrix recovery~\cite{li2018algorithmic, zhuo2021computational, zhang2019sharp, zhang2021sharp}.

Despite the significance of different notions of RIP within the realm of low-rank matrix recovery, they face major breakdowns in {robust} settings, where a subset of the measurements are grossly corrupted with large noise values. The main reason behind the failure of the existing RIP techniques is that they only apply to nearly clean measurements, and hence, are oblivious to the nature of the noise.  


The main goal of this paper is to precisely pinpoint and remedy this challenge. In particular, we study a well-known class of matrix recovery problems with {nonsmooth and nonconvex $\ell_1$ formulation}, called \textit{robust matrix recovery}. We introduce an alternative notion of RIP, called \textit{Sign-RIP}, that can capture and take into account the nature of the noise in robust matrix recovery. Based on Sign-RIP, we take the first step towards demystifying the robustness of the $\ell_1$ formulation of the problem against large-and-sparse noise values. 
Our main contributions are summarized as follows:
\begin{itemize}
\item[-] \textit{(Uniform convergence of subdifferentials)} We use Sign-RIP to study the landscape of the robust matrix recovery against large noise values. In particular, we show that, under Sign-RIP, the subdifferentials of the nonsmooth matrix recovery are well-behaved, and they converge uniformly to the gradients of an ``ideal'', noiseless problem, even if the measurements are subject to {large} noise values. Moreover, we show that Sign-RIP holds, even if an {arbitrarily large} fraction of the measurements are corrupted with {arbitrarily large} noise values, provided that the number of measurements scales polynomially with the corruption probability, but only {linearly} with the true dimension of the problem. 
    \item[-] \textit{(Characterization of the critical points)} We show that Sign-RIP can be used to precisely characterize the locations of the critical points for the robust rank-1 matrix recovery. In particular, we show that, under Sign-RIP, all the critical points lie close to the true rank-1 solution, or have small norm. 
    \item[-] \textit{(Implicit regularization with over-parameterization)} Based on Sign-RIP, we show that a subgradient method with decaying step sizes provably converges to the true rank-1 solution in the over-parameterized regime, where the true rank is unknown and over-estimated.  
\end{itemize}

\paragraph{Notations}
For two matrices $X$ and $Y$ of the same size, their inner product is defined as $\inner{X}{Y}=\mathrm{Tr}(X^{\top}Y)$. 
For a matrix $X$, its operator and Frobenius norms are denoted as $\norm{X}$ and $\norm{X}_F$, respectively. The unit rank-$r$ sphere is defined as $\mathbb{S}_r=\{X\in \R^{d\times d}:\norm{X}_F=1, \rank(X) \leq r\}$. The notation $\mathbb{B}(X,\epsilon)$ refers to a ball of radius $\epsilon$, centered at $X$. Inspired by the variational representation of Frobenius norm, we define $\|X\|_{F,r} = \max_{Y\in\mathbb{S}_r}\langle X, Y\rangle$. It is easy to verify that $\|X\|_{F} = \|X\|_{F,d}$.
The $\ell_q$ norm of a vector $x$ is defined as $\norm{x}_{\ell_q}=(\sum |x_i|^q)^{1/q}$. For simplicity of notation,  we write $\norm{x}_1=\norm{x}_{\ell_1}$ and $\norm{x}=\norm{x}_{\ell_2}$. Given two sequences $f(n)$ and $g(n)$, the notation $f(n)\lesssim g(n)$ implies that there exists a constant $C<\infty$ satisfying $f(n) \leq Cg(n)$. Moreover, the notation $f(n)\asymp g(n)$ implies that $f(n)\lesssim g(n)$ and $g(n)\lesssim f(n)$. The sign function $\sign(\cdot)$ is defined as $\sign(x)=x/|x|$ if $x\neq 0$, and $\sign(0)=[-1,1]$. 

\section{Background and Prior Work}\label{sec_RIP}
In robust matrix recovery problem, the goal is to recover a rank-$r^*$ positive semidefnite matrix $X^*\in\mathbb{R}^{d\times d}$, from a limited number of linear measurements of the form $\mathbf{y} = \mathcal{A}(X^*)+\mathbf{s}$, where $\mathbf{y}=[y_1,y_2,\dots,y_m]^\top$ is the vector of measurements, and $\mathbf{s}$ is a noise vector. The linear operator $\mathcal{A}$ is defined as $\mathcal{A}(X^*) = [\langle A_1,X^*\rangle, \langle A_2,X^*\rangle,\dots, \langle A_m,X^*\rangle]^\top$, where $\{A_i\}_{i=1}^m$ are symmetric measurement matrices\footnote{We can replace $A_i$ by $\frac{A_i+A_i^{\top}}{2}$ if $A_i$ is asymmetric. For simplicity, we only consider symmetric measurement matrices in this work.}. One popular approach for recovering the true low-rank matrix is to consider the following empirical risk minimization (ERM) problem
\begin{equation}\label{eq_l2}
    \min_{U\in\R^{d\times r'}} f_{\ell_q}(U)=\frac{1}{2m}\left\|{\mathbf{y}-\mathcal{A}(UU^{\top})}\right\|^q_{\ell_q},
\end{equation}
where $r'$ is an upper bound for the rank of the true solution, and $UU^\top$ is used in lieu of $X^*$ to ensure the positive semidefiniteness of the solution. 

{\bf $\pmb{\ell_2}$-RIP:} Evidently, the above optimization problem is over-parameterized if $r'> r^*$, since the unknown variable is not restricted to the set of low-rank matrices, and consequently, its globally optimal solution need \textit{not} be low-rank. Nonetheless, it is recently shown that, for the choice of $q=2$, simple gradient descent (GD) algorithm provably converges to the true rank-$r^*$ solution, even if $r'\gg r^*$ (e.g. $r'=d$)~\cite{li2018algorithmic, zhuo2021computational}. The key idea behind the convergence proof of GD is the closeness of its gradient to that of an ``ideal'', noiseless population loss function $\bar{f}_{\ell_2}(U) = \|UU^\top-X^*\|_F^2$. More concretely, the gradient of $f_{\ell_2}(U)$ can be written as $\nabla f_{\ell_2}(U) = Q(UU^\top-X^*)U$, where $$Q(M) = \frac{1}{m}\sum_{i=1}^m\left(\langle A_i, M\rangle+s_i\right) A_i.$$ 
One sufficient condition for $\nabla f_{\ell_2}(U)\approx \nabla \bar f_{\ell_2}(U)$ is to ensure that $Q(M)$ remains uniformly close to $M$ for every rank-$(r^*+r')$ matrix $X$.\footnote{The paper~\cite{li2018algorithmic} requires the similarity of $Q(X)$ and $X$ for lower rank matrices (rank-$r^*$ as opposed to rank-$(r^*+r')$), but their result only holds for $r'=d$.} In the noiseless setting, this condition can be guaranteed by the so-called $\ell_2$-RIP:

\begin{definition}[$\ell_2$-RIP~\cite{zhuo2021computational, li2018algorithmic}]
\label{def::l2-RIP}
The linear operator $\mathcal{A}(\cdot)$ satisfies $\ell_2$-RIP with parameters $(r,\delta)$ if, for every rank-$r$ matrix $M$, we have
$$(1-\delta)\|M\|_F^2\leq \frac{1}{m}\|\mathcal{A}(M)\|^2\leq (1+\delta)\|M\|_F^2.$$

\end{definition}
Roughly speaking, $\ell_2$-RIP entails that the linear operator $\mathcal{A}(\cdot)$ is nearly norm-preserving for every rank-$r$ matrix. It is well-known that with Gaussian measurements, $\ell_2$-RIP is satisfied with parameters $(r,\delta)$, provided that $m\gtrsim dr$~\cite{recht2010guaranteed}.
However, our next proposition shows that $\ell_2$-RIP is not enough to guarantee $Q(X)\approx X$ when the measurements are subject to noise with high variance.
\begin{proposition}
    \label{uniform-convergence-noisy-l2}
    Suppose that $r'=d$ and the measurement matrices $\{A_i\}_{i=1}^m$ defining the linear operator $\mathcal{A}(\cdot)$ have i.i.d. standard Gaussian entries. Moreover, suppose that the noise vector $\mathbf{s}$ satisfies $s_i\stackrel{i.i.d.}{\sim} \mathcal{N}(0,\sigma^2)$ with probability $p$, and $s_i=0$ with probability $1-p$, for every $i=1,\dots,m$. Then, $\ell_2$-RIP is satisfied with an overwhelming probability with parameters $(d,\delta)$, provided that $m\lesssim d^2/\delta^2$. Moreover, we have
    \begin{equation}
        \mathbb{P}\left(\sup_{X\in\mathbb{S}}\left\|{Q(X)-X}\right\|_F\gtrsim \sqrt{\frac{(1+p\sigma^2)d^2}{m}}\right)\geq \frac{1}{2}.\nonumber
    \end{equation}
\end{proposition}
\begin{wrapfigure}{r}{8.5cm}
		\vspace{-7mm}
		\subfloat[]{%
		\includegraphics[width=4.1cm]{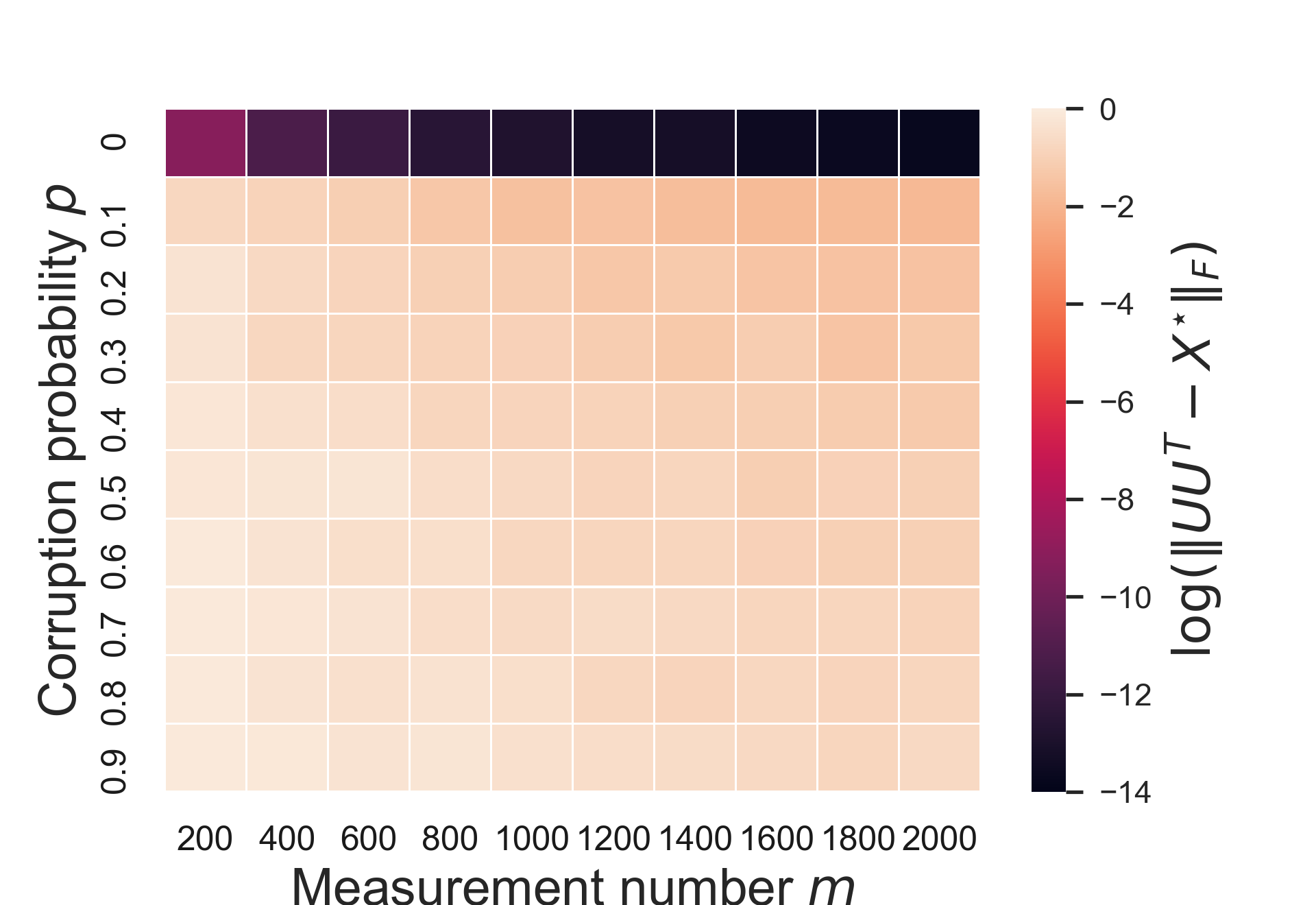}\label{fig_l2}%
	}
	\subfloat[]{%
		\includegraphics[width=4.4cm]{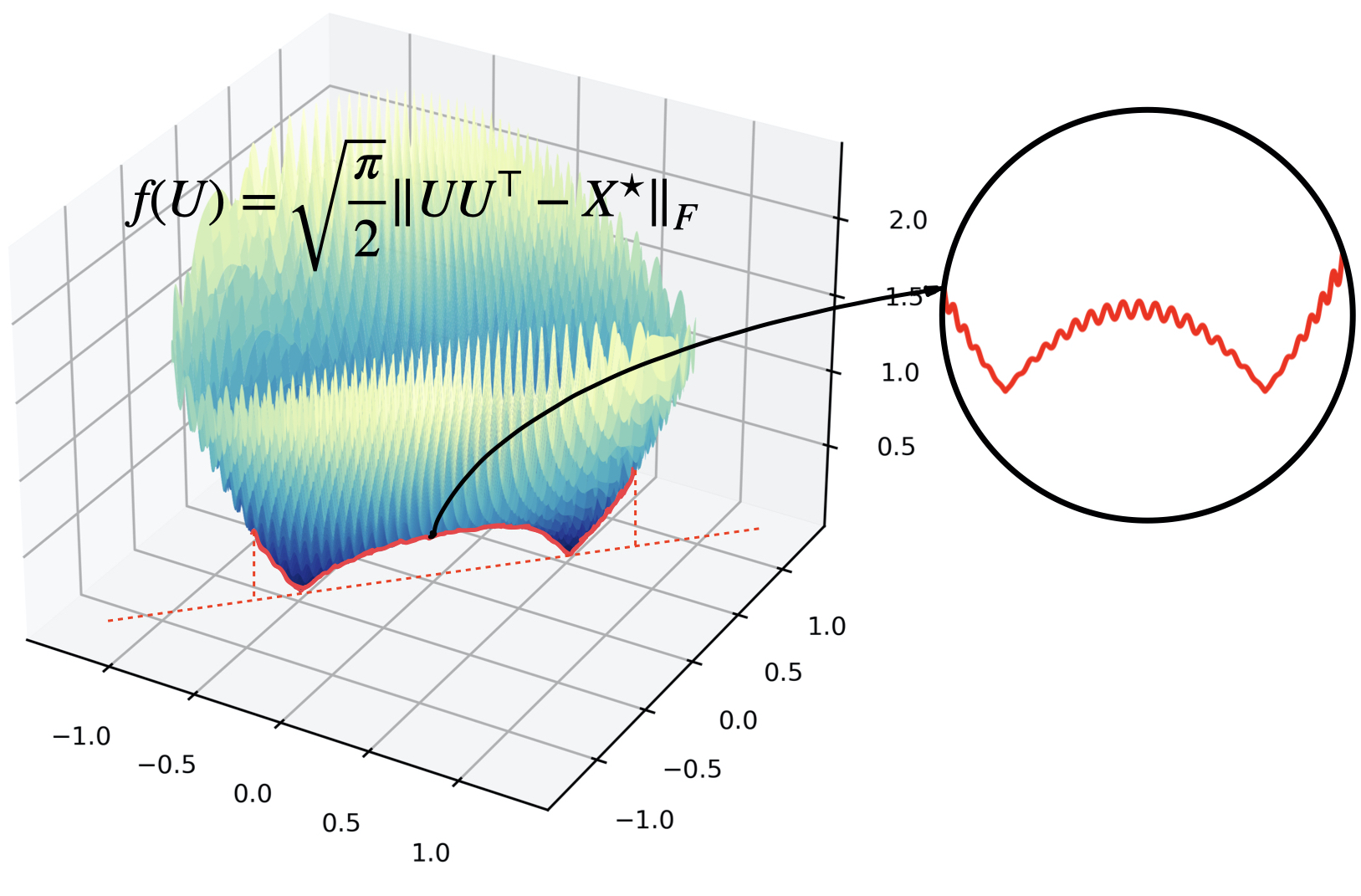}\label{fig_l1-l2}%
	}
	\caption{\footnotesize (a) The accuracy of the obtained solutions by solving~\eqref{eq_l2} via gradient descent. Each measurement is corrupted with a large noise with probability $p$. (b) The linear operator $\mathcal{A}(\cdot)$ satisfies the $\ell_1/\ell_2$-RIP, but the subdifferentials of the $\ell_1$-loss show sporadic behavior. Here, $f(U)$ is the population loss.}
		\vspace{-5mm}
\end{wrapfigure}

On one hand, the above proposition shows that, in order to guarantee $Q(X)\approx X$, the number of measurements should be at least $m\gtrsim (1+p\sigma^2)d^2$, and hence, grow with the variance of the noise. 
On the other hand, for any fixed $\delta$, $\ell_2$-RIP is guaranteed to be satisfied with $m\lesssim d^2$, which is independent of the noise variance. This highlights a fundamental pitfall of $\ell_2$-RIP in the face of large noise values: the matrices $Q(X)$ and $X$ may be far apart, \textit{even if} the linear mapping $\mathcal{A}(\cdot)$ satisfies $\ell_2$-RIP. Figure~\ref{fig_l2} illustrates that the discrepancy between $Q(X)$ and $X$ in the noisy setting leads to the ultimate failure of the gradient descent algorithm.
\paragraph {\bf $\pmb{\ell_1/\ell_2}$-RIP:} 
With the goal of robustifying the solution against outlier noise, recent work
has studied the landscape of the nonsmooth optimization~\eqref{eq_l2}
with $q=1$ under $\ell_1/\ell_2$-RIP. Roughly speaking, $\ell_1/\ell_2$-RIP imposes a similar condition to $\ell_2$-RIP, but on the $\ell_1$-loss function. 
In particular, it entails that $\frac{1}{m}\|\mathcal{A}(X)\|_1$ remains close to $\sqrt{2/\pi}\|X\|_F$, for every rank-$r$ matrix $X$. Under $\ell_1/\ell_2$-RIP, it is recently shown that subgradient method converges to the ground truth with $q=1$, provided that the true rank of the solution is known and the initial point is sufficiently close to the ground truth~\cite{li2020nonconvex, tong2021low}. However, $\ell_1/\ell_2$-RIP is also oblivious to the nature of the noise, and as a result, cannot guarantee the global convergence of the corresponding subdifferentials of $f_{\ell_1}(U)$. Figure~\ref{fig_l1-l2} shows an instance of~\eqref{eq_l2} with $q=1$, where the linear operator $\mathcal{A}(\cdot)$ satisfies $\ell_1/\ell_2$-RIP, and yet the subdifferentials of the loss function suffer from sporadic behavior due to noise, giving rise to numerous undesirable local minima.


The aforementioned challenges highlight the fundamental limitations of the existing notions of RIP in the context of robust matrix recovery with large noise. This calls for a new approach for analyzing the landscape of robust matrix recovery; a goal that is at crux of this paper.

\paragraph{Other works.} Other variants of robust matrix recovery have been studied before. \citet{fattahi2020exact} and \citet{josz2018theory} prove that~\eqref{eq_l2} with $r'=r=1$ and $q=1$ has no spurious local solution, provided that the measurement matrices correspond to element-wise projection operators. In the over-parameterized regime, \citet{you2020robust} propose to circumvent the nonsmoothness of the robust matrix recovery problem by resorting to a smooth doubly over-parameterized model, and then solving it via gradient descent.

Moreover, the concentration of subdifferentials plays a crucial role in the convergence analysis of first-order methods in ERM problems. Recently, \citet{mei2018landscape} and \citet{foster2018uniform} provide uniform gradient bounds for nonconvex and smooth ERM problems. However, these works heavily rely on the smoothness of the loss function. Within the realm of nonsmooth optimization, \citet{bai2018subgradient} studies robust dictionary learning problem with $\ell_1$-loss function, and proposes a uniform bound on its subdifferential with respect to Hausdorff distance. 
Moreover, \citet{davis2018graphical} obtain a dimension-dependent bound for a similar problem via a smoothing technique. However, these techniques are not directly applicable to the robust matrix recovery.

\section{Our Approach: Sign-RIP for $\pmb{\ell_1}$-loss Function}
Our goal is to study the following ERM problem with $\ell_1$-loss function:
\begin{equation}\label{eq_l1}
    \min_{U\in\R^{d\times r'}} f_{\ell_1}(U)=\frac{1}{2m}\left\|{\mathbf{y}-\mathcal{A}(UU^{\top})}\right\|_{1},
\end{equation}

\begin{algorithm}[tb]
   \caption{Subgradient Method}
   \label{algorithm}
\begin{algorithmic}
  \STATE {\bfseries Input:} measurement matrices $\{A_i\}_{i=1}^m$, measurement vector $\mathbf{y}=[y_1,\cdots,y_m]^\top$, number of iterations $T$, an upper bound on the rank $r'$, and an initialization matrix $B_0\in\mathbb{R}^{d\times r'}$;
   \STATE {\bfseries Output:} Solution $\hat{X}_T=U_T U^{\top}_T$ to~\eqref{eq_l1};
   \STATE Initialize $U_0=B_0$.
   \FOR{$t\leq T$}
   \STATE Compute a subgradient $D_t\in \partial f_{\ell_1}(U_t)$;
   \STATE Select the step size $\eta_t$;
   \STATE Set $U_{t+1}\leftarrow U_t - \eta_t D_t$;
   \ENDFOR
\end{algorithmic}
\end{algorithm}
The simplest algorithm for solving~\eqref{eq_l1} is subgradient method (SubGD). At every iteration, SubGD selects an arbitrary direction $D_t$ from the subdifferential of the $\ell_1$-loss at the current solution, and then updates the solution by moving towards $-D_t$ with a step size $\eta_t$; see Algorithm~\eqref{algorithm}. Figure~\ref{fig_iter} illustrates that SubGD with diminishing step sizes can successfully recover the true rank-1 solution in both exact and over-parameterized regimes when applied to the $\ell_1$-loss~\eqref{eq_l1}, even if $10\%$ of the measurements are grossly contaminated with noise. On the other hand, GD on the smooth loss function quickly overfits to the noise within a few iterations. Figure~\ref{fig_heatmap} shows the robustness of SubGD against increasing fraction of noisy measurements. It can be seen that SubGD recovers the true solution, even if more than half of the measurements are corrupted with noise,  pinpointing its superiority over GD.

To study the convergence of SubGD, it is essential to analyze the behavior of the subdifferential $\partial f_{\ell_1}(U_t)$, which can be written as 
\begin{align}
    \partial f_{\ell_1}(U_t) = \frac{1}{m}\sum_{i=1}^m\sign(\langle A_i,U_t U_t^{\top}-X^{\star}\rangle-s_i)A_i U_t.\nonumber
\end{align}

\begin{figure*}
\begin{center}
\subfloat[]{
{\includegraphics[width=5.5cm]{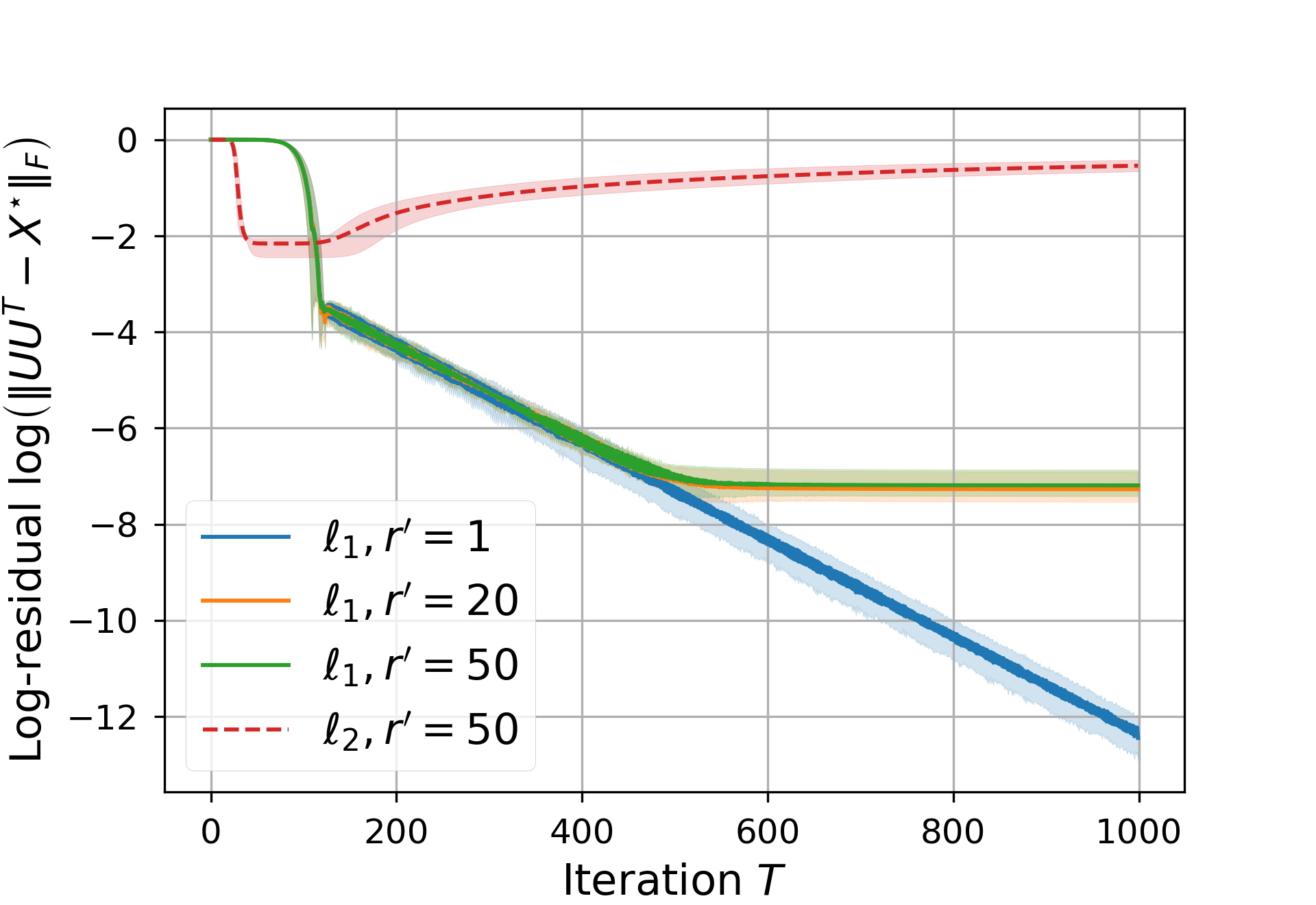}}\label{fig_iter}}
\subfloat[]{
{\includegraphics[width=5.5cm]{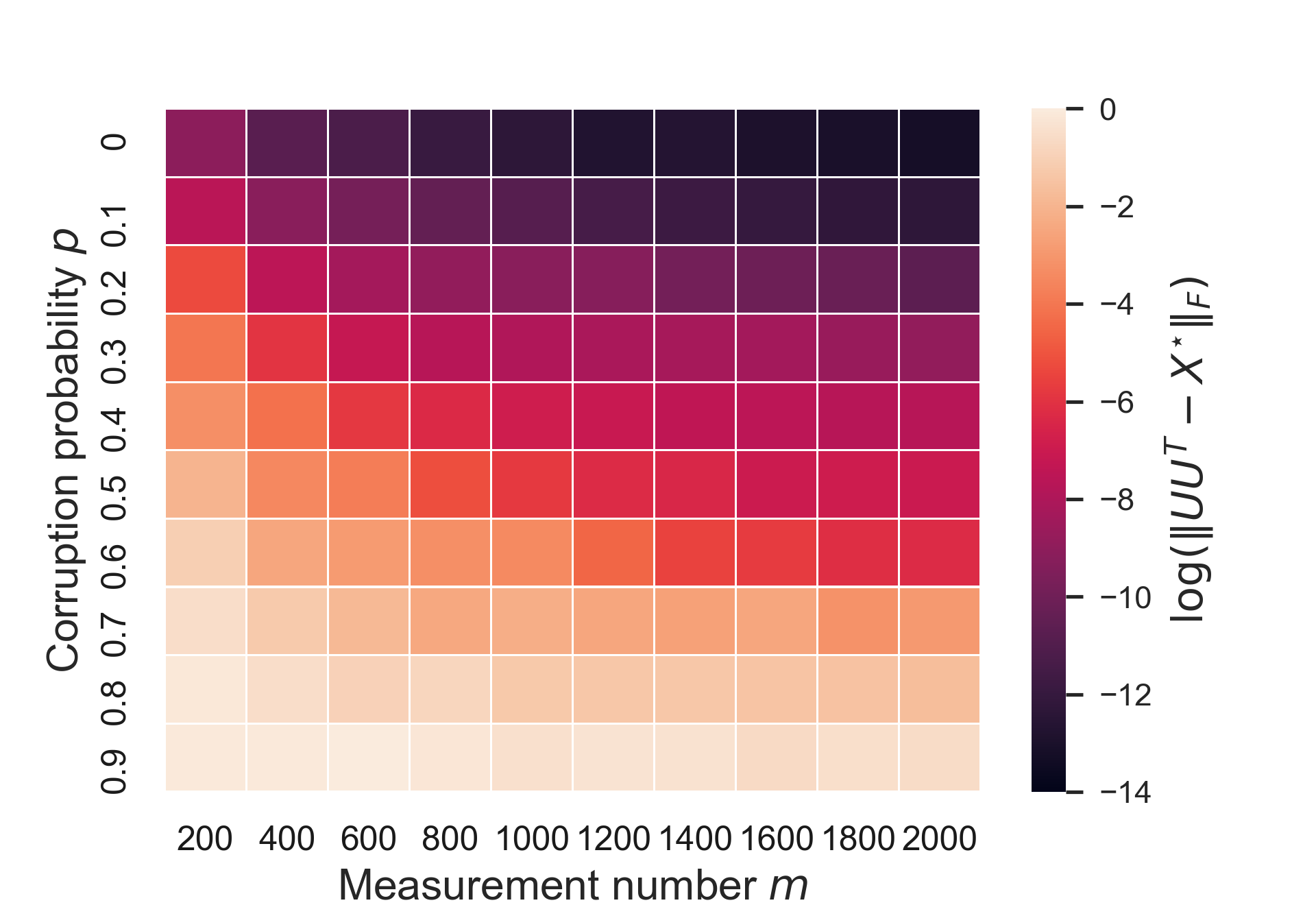}}\label{fig_heatmap}}
\end{center}
\caption{\footnotesize (a) SubGD with geometrically decaying step sizes recovers the true rank-$1$ matrix in both exact ($r'=r=1$) and over-parameterized ($r'>r$) regimes. GD with small step sizes overfits to the noise. (b) SubGD with geometrically decaying step sizes recovers the true rank-$1$ matrix in the over-parameterized regime ($r'=d = 20$), for different number of measurements $m$ and corruption probabilities $p$.}
\label{fig_l1}
\end{figure*}
The key idea behind our proposed technique is to study the theoretical superiority of the robust matrix recovery with $\ell_1$-loss, by characterizing its distance to an ideal, noiseless problem $\bar f_{\ell_2}(U) = \|UU^\top-X^*\|_F^2$.
In particular, we show that robust matrix recovery with $\ell_1$-loss benefits from strong robustness and algorithmic properties, provided that its subdifferentials resemble the gradients of $\bar f_{\ell_2}(U)$. In particular, upon defining $\mathcal{Q}(X) = \frac{1}{m}\sum_{i=1}^m\sign(\langle A_i,X\rangle-s_i)A_i$, our goal is to provide conditions under which, for \textit{any} $Q\in\mathcal{Q}(UU^\top-X^*)$, the subgradient $QU$ of the $\ell_1$-loss function points {approximately} towards the direction $(UU^\top-X^*)U$, which is the gradient of $\bar f_{\ell_2}(U)$. This is formalized through the notion of \textit{Sign-RIP}, which is defined below. 

\begin{definition}[Sign-RIP]\label{def_sign_RIP}
The measurements are said to satisfy \textit{Sign-RIP} with parameters $(r,\delta)$ and a scaling function $\varphi:\mathbb{R}^{d\times d}\to\mathbb{R}$ if, for every rank-$r$ matrix $X$ and every $Q\in\mathcal{Q}(X)$, we have
\begin{align}\label{eq_sign_RIP}
    \left\|Q-\varphi(X)\frac{X}{\|X\|_F}\right\|_{F,r}\leq \varphi(X)\delta.
\end{align}

\end{definition}
At the first glance, one may speculate that Sign-RIP is extremely restrictive: roughly speaking, it requires the uniform concentration of random set-valued function $\mathcal{Q}(X)$, for \textit{every} rank-$r$ matrix $X$. However, we show that, for a fixed $\delta$, Sign-RIP enjoys the same linear sample complexity as $\ell_2$-~\cite{recht2010guaranteed} and $\ell_1/\ell_2$-RIP~\cite{li2020nonconvex}, and hence, is not statistically more restrictive than its classical counterparts. In particular, we show that Sign-RIP holds with high probability with linear number of Gaussian measurements, even if an {arbitrarily large} fraction of them are corrupted with large noise values. To this goal, first we introduce our considered noise model:

\begin{assumption}[Noise model]\label{assump_noise}
    Given a corruption probability $p$, the noise vector $\mathbf{s}\in \R^{m}$ is generated as follows: first, a subset $\mathcal{S}\subset \{1,2,\dots,m\}$ with cordiality $pm$ is selected uniformly at random. Then, for every $i\in \mathcal{S}$, the element $s_i$ is randomly drawn from a zero mean distribution, i.e., $s_i\stackrel{i.i.d.}{\sim} \mathbb{P}$ with $\mathbb{E}[s_i] = 0$ and arbitrary variance. Finally, $s_i = 0$ for every $i\not\in\mathcal{S}$.
    \label{assumption-noise}
\end{assumption}
\begin{remark}
Our subsequent results hold under an alternative (and equivalent) noise model studied in \cite{bai2018subgradient}, where $s_i\stackrel{i.i.d.}{\sim} \mathbb{P}$ with probability $p$, and $s_i=0$ with probability $1-p$.
\end{remark}
Notice that our proposed noise model does not impose any assumption on the magnitude of the nonzero elements of $\mathbf{s}$, or the specific form of their distribution, which makes it particularly suitable for modeling outliers with arbitrary magnitudes. 

Next, we show that, with sufficiently large number of Gaussian measurements, Sign-RIP is satisfied with an appropriate choice of scaling function, even if a constant fraction of the measurements are corrupted with noise.
\begin{theorem}\label{thm_G_noisy}
\begin{sloppypar}
Assume that the measurement matrices $\{A_i\}_{i=1}^m$ defining the linear operator $\mathcal{A}(\cdot)$ are symmetric with i.i.d. standard Gaussian entries, and that the noise vector $\mathbf{s}$ satisfies Assumption~\ref{assump_noise}. Then, Sign-RIP holds with parameters $(r,\delta)$ and a scaling function $\varphi(X) = \sqrt{\frac{2}{\pi}}\left(1-p+p\mathbb{E}\left[e^{-s_i^2/(2\|X\|_F)}\right]\right)$ with probability of at least $1-Ce^{-cm\delta^4}$, provided that $m\gtrsim \frac{dr\left(\log\left(\frac{1}{(1-p)\delta}\right)\vee 1\right)}{\delta^4(1-p)^4}$.
\end{sloppypar} 
\end{theorem}

A number of observations can be made based on Theorem~\ref{thm_G_noisy}. First, it entails that $QU \approx \alpha (UU^\top - X^*) U$ for some $\alpha>0$ and every $Q\in\mathcal{Q}(UU^\top-X^*)$. This implies that the search direction $D_t$ of SubGD is approximately aligned with the descent direction of the ideal, noiseless loss function $\bar{f}_{\ell_2}(U) = \|UU^\top-X^*\|_F^2$, even if a fraction of the measurements are severely corrupted with noise. Second, it shows that, for any fixed corruption probability $p$ and parameter $\delta$, Sign-RIP is satisfied with $\mathcal{O}(dr)$ number of Gaussian measurements, which is the same as those of $\ell_2$-~\cite{recht2010guaranteed} and $\ell_1/\ell_2$-RIP~\cite{li2020nonconvex}. Moreover, our result does not impose any restriction on the corruption probability $p$, which improves upon the assumption $p<1/2$ made in~\citep{li2020nonconvex, tong2021low} for the robust matrix recovery problem. Finally, our result holds {irrespective} of the magnitude of the noise values, ultimately alleviating the issue raised in Subsection~\ref{sec_RIP}.

{\bf Sketch of the proof for Theorem~\ref{thm_G_noisy}.}  The details of our derivations can be found in Appendix B. Here, we provide a short and informal overview of the key ingredients of our proof technique. To prove Theorem~\ref{thm_G_noisy}, we must establish the uniform convergence of $\frac{1}{m}\sum_{i=1}^{m}\sign(\inner{A_i}{X})\inner{A_i}{Y}$ over a low-rank manifold.
This can be done by showing the uniform concentration of $\sign(\inner{g}{x})\inner{g}{y}$, where $g$ is an i.i.d. standard Gaussian vector, and $x,y$ are the vectorized versions of $X,Y\in\mathbb{S}_r$. The previous proof techniques for establishing $\ell_2$- and $\ell_1/\ell_2$-RIP heavily rely on the Lipschitzness of the underlining functions, which does not hold in our problem due to the discontinuity of the $\sign$ function.
To address this challenge, we apply a novel {discretization} argument based on the following decomposition 
\begin{equation}
    \sign(\inner{g}{x})\inner{g}{y} = \sign(\inner{g}{\pi(x)})\inner{g}{y} + \left(\sign(\inner{g}{x})-\sign(\inner{g}{\pi(x)})\right)\inner{g}{y}.\nonumber
\end{equation}
Here $\pi(x)$ is the closest element to $x$ in the $\varepsilon$-net of $\mathbb{S}_r$. The first term can be easily controlled due to the small cardinality of the $\varepsilon$-net of $\mathbb{S}_r$. On the other hand, the second term can be uniformly bounded by characterizing the deviation of the $\sign$ function from the linear part
\begin{equation}
    \left(\sign(\inner{g}{x})-\sign(\inner{g}{\pi(x)})\right)\leq \left|\sign(\inner{g}{x})-\sign(\inner{g}{\pi(x)})\right|\cdot |\inner{g}{y}|.\nonumber
\end{equation}
Since we {localize} the discontinuous $\sign$ function, we can utilize the Talagrand-type inequalities~\citep{talagrand1996new, sen2018gentle} to show that $\sign(\inner{g}{x})-\sign(\inner{g}{\pi(x)})$ is uniformly small. This, together with a union bound leads to the final result.$\hfill\square$


Equipped with Sign-RIP, we next study the landscape of the robust rank-1 matrix recovery with outlier noise. 

\section{Characterization of Critical Points in Robust Rank-$1$ Matrix Recovery}

In this section, we characterize the critical points of the robust rank-1 matrix recovery. In particular, we show that, under Sign-RIP, all critical points lie within a small neighborhood of the ground truth or the origin. 

Suppose that $r'=r^{\star}=1$, and $X^*=u^*{u^*}^\top$ for $u^*\in\mathbb{R}^{d\times 1}$. For simplicity and without loss of generality, we assume that $\|u^*\|= 1$. Recall that $\bar U$ is a \textit{critical point} of $f_{\ell_1}(U)$ if it satisfies $0\in\partial f_{\ell_1}(\bar U)$. Moreover, $\bar U$ is a \textit{local minimum} of $f_{\ell_1}(U)$ if $f_{\ell_1}(\bar U)\leq f_{\ell_1}(U)$ for every $U\in \mathbb{B}(\bar U, \epsilon)$, for some $\epsilon>0$. All local minima of $f_{\ell_1}(U)$ are also critical points~\cite{clarke1990optimization}. Our next theorem characterizes the critical points of $f_{\ell_1}(U)$, the proof of which is in Appendix C of the supplementary material.
\begin{theorem}[Critical Points]\label{thm_benign}
Assume that the measurements satisfy the Sign-RIP condition with parameters $\left(2,\delta\right)$ and a strictly positive and uniformly bounded scaling function $\varphi(X)$. Moreover, suppose that $U$ with $\|U\|\leq R$ for some $R\geq 1$ is a critical point of~\eqref{eq_l1} with $r'=r^*=1$. Then, we have $\|UU^\top-u^*{u^*}^\top\|_F\lesssim\delta$ or $\|U\|^2\lesssim\delta$, provided that $\delta\lesssim1/R^3$.
\end{theorem}
The above theorem shows that, under Sign-RIP with small $\delta$, the critical points of $f_{\ell_1}(U)$ are either close to the ground truth, or have very small norm. 
Combined with Theorem~\ref{thm_G_noisy}, this leads to the following corollary.
\begin{corollary}\label{cor_benign}
Assume that the measurement matrices $\{A_i\}_{i=1}^m$ defining the linear operator $\mathcal{A}(\cdot)$ are symmetric with i.i.d. standard Gaussian entries, and that the noise vector $\mathbf{s}$ satisfies Assumption~\ref{assump_noise}. Moreover, suppose that $U$ with $\|U\|\leq R$ for some $R\geq 1$ is a critical point of~\eqref{eq_l1} with $r'=r^*=1$. Then, we have $\|UU^\top-u^*{u^*}^\top\|_F\lesssim\delta$ or $\|U\|^2\lesssim\delta$ with an overwhelming probability, provided that $m\gtrsim \frac{dr'\log\left(\frac{1}{(1-p)\delta}\right)}{\delta^4(1-p)^4}$ and $\delta\lesssim 1/R^3$. Additionally, we have $UU^\top = u^*{u^*}^\top$ or $\|U\|\lesssim\delta$, if $p\leq \frac{1}{2}-\frac{\delta}{\sqrt{2/\pi}-\delta}$. 
\end{corollary}

Earlier works on robust matrix recovery with $\ell_1$-loss can only characterize the critical points of $f_{\ell_1}(U)$ locally within a very {small} neighborhood of the global minima~\cite{li2020nonconvex, tong2021low}. Corollary~\ref{cor_benign} extends this result in two ways for $r^*=1$: first, it provides a \textit{global} characterization of the critical points. In particular, it shows that, with sufficiently large number of measurements, all critical points with a bounded norm concentrate around the ground truth or the origin, provided that $p<1$. Moreover, it shows that, if additionally we have $p<1/2$, all critical points that are not close to the origin must coincide the ground truth.

\section{Over-parameterized Robust  Rank-$1$ Matrix Recovery}
In this section, we study the over-parameterized robust matrix recovery problem, where the rank of the true solution is unknown and over-estimated. In particular, we show that,
{\it under Sign-RIP condition, SubGD converges to the true rank-1 solution in the over-parameterized regime, without any explicit regularization or rank constraint.}

\noindent{\bf Intuition behind our analysis.} Before delving into the details, we shall provide the intuition behind our analysis. Suppose that Sign-RIP holds with sufficiently small $\delta$. Then, we have $D_t\approx \varphi_t\frac{(U_tU_t-X^*)U_t}{\|U_tU_t-X^*\|_F}$ for every $D_t\in\partial f_{\ell_1}(U_t)$, where for simplicity, we define $\varphi_t = \varphi(U_tU_t-X^*)$. Based on this approximation, the iterations of SubGD can be approximated as $U_{t+1} \approx U_t - \eta_t\varphi_t\cdot\frac{(U_tU_t-X^*)U_t}{\|U_tU_t-X^*\|_F}$.
Consequently, with the choice of $\eta_t = \eta_0\varphi_t^{-1}\|U_tU_t-X^*\|_F$, the iterations of SubGD reduce to
\begin{equation}\label{eq_stepsize}
    U_{t+1} \approx U_t - \eta_0\cdot{(U_tU_t-X^*)U_t}
\end{equation}
which are precisely the iterations of GD with a constant step size $\eta_0$, applied to the ideal, noiseless $\ell_2$-loss function $\bar{f}_{\ell_2}(U) = \|UU^\top-X^*\|_F^2$. This implies that, under Sign-RIP, SubGD behaves similar to GD with a constant step size, when applied to $\bar{f}_{\ell_2}(U)$.  A caveat of this analysis is that the proposed step sizes are in terms of $\varphi_t^{-1}\|U_tU_t-X^*\|_F$, which is not known \textit{a priori}. In the noiseless scenario, Sign-RIP can be invoked to show that $\varphi_t^{-1}\|U_tU_t-X^*\|_F$ can be accurately estimated as $\frac{\pi}{2m}\|\mathbf{y}-\mathcal{A}(U_tU_t)\|_1$. However, with noisy measurements, the value of $\|U_tU_t-X^*\|_F$ \textit{cannot} be estimated merely based on $\|\mathbf{y}-\mathcal{A}(U_tU_t^\top)\|_1$, since the $\ell_1$-loss function $\|\mathbf{y}-\mathcal{A}(U_tU_t^\top)\|_1$ is no longer an unbiased estimator of $\|U_tU_t^\top-X^*\|_F$ and is highly sensitive to the magnitude of the noise. To alleviate this issue, we propose the following alternative choice of step size:
\begin{align}\label{eq_stepsize2}
    \eta_t = \frac{\eta_0}{\|Q\|_F}\rho^t
\end{align}
where $Q\in\mathcal{Q}(U_tU_t-X^*)$, and $0<\rho<1$ is a predefined decay rate. Due to Sign-RIP, we have $\|Q\|_F\approx \varphi(U_tU_t^\top-X^*)$, which implies 
\begin{align}\label{eq_iter}
        U_{t+1}
        \approx U_t-\eta_0\rho^t \frac{\left(U_tU_t^{\top}-X^{\star}\right)U_t}{\norm{U_tU_t^{\top}-X^{\star}}_F}.
\end{align}
A closer look at~\eqref{eq_iter} reveals that, at every iteration, SubGD points to the negative gradient of $\|U_tU_t^\top-X^*\|_F^2$, while the geometrically decaying step size $\eta_t = \eta_0\rho^t$ guarantees the convergence of the algorithm.

\begin{algorithm}[tb]
   \caption{Spectral Initialization}
   \label{alg::spectral-initialization}
\begin{algorithmic}
  \STATE {\bfseries Input:} measurement matrices $\{A_i\}_{i=1}^m$, measurement vector $\mathbf{y}=[y_1,\cdots,y_m]^\top$,  an upper bound on the rank $r'$, and an initial scaling factor $\alpha$;
   \STATE {\bfseries Output:} An initialization matrix $B_0\in \R^{d\times r^{\prime}}$;
   \STATE Calculate $C=\frac{1}{m}\sum_{i=1}^{m}\sign(y_i){A_i}$, and its normalized variant $\hat{X}=C/\norm{C}_F$;
   \STATE Compute the eigenvalue decomposition $\hat{X}=V\Sigma V^{\top}$;
   \STATE Define $\Sigma_{+}^{r^\prime}$ as the top $r^{\prime}\times r^{\prime}$ sub-matrix of $\Sigma$ corresponding to $r'$ largest eigenvalues of $\hat{X}$, whose negative values are replaced by $0$;
   \STATE Set $B_0 = \alpha V \left(\Sigma_{+}^{r^\prime}\right)^{1/2}$.
\end{algorithmic}
\end{algorithm}

Inspired by this intuition, we provide our main result.

\begin{theorem}
    \label{convergence-theorem-noisy}
     Assume that $r'\geq r^* = 1$, and the measurements satisfy the Sign-RIP condition with parameters $\left(\min\{r'+1,d\},\delta\right)$, where $\delta\lesssim 1$ and $\varphi(X)$ is strictly positive and uniformly bounded. Suppose that the initialization matrix $B_0$ is chosen via Algorithm~\ref{alg::spectral-initialization}, and $U_T$ is obtained via Algorithm~\ref{algorithm}.
    Moreover, suppose that $\alpha\asymp\sqrt{\delta}/\sqrt[4]{r'}$, and the step size $\eta_t$ is chosen as~\eqref{eq_stepsize2} with $\eta_0\lesssim \delta$ and $\rho= 1-\Theta\left(\eta_0/\log\frac{1}{\alpha}\right)$. Then, after $T\asymp\log\left(\frac{r'}{\delta}\right)/\eta_0$ iterations, we have
    \begin{equation}\label{eq_error}
        \norm{U_TU_T^{\top}-X^{\star}}^2_F\lesssim \delta^2\log^2\left(\frac{r'}{\delta}\right).
    \end{equation}
\end{theorem}
The above theorem implies that, for any $r'\geq r^* = 1$ (including $r'=d$), SubGD converges to the true low-rank solution at a {sublinear} and (nearly) {dimension-free} rate without any explicit regularization or rank constraint, provided that the measurements satisfy Sign-RIP. Moreover, our result holds under the so-called \textit{early stopping} regime, where the number of iterations of SubGD is both upper and lower bounded by problem-specific parameters. Similar requirements are also imposed in other over-parameterized problems~\citep{gunasekar2018implicit, li2018algorithmic}. Finally, we conjecture that the early stopping is an artifact of our proof technique, and it is not necessary due to the geometrically decaying step sizes. We consider the rigorous verification of this conjecture as an enticing challenge for future research.

\begin{remark}
Our convergence result is independent of the scaling function $\varphi(X)$. This is due to the special choice of the step size: roughly speaking, Sign-RIP implies that the chosen step size is proportional to $\varphi_t^{-1}$, thereby cancelling the effect of the scaling function in the dynamics of SubGD. This implies that the step sizes of our algorithm are adaptive to the corruption probability, whose effect is captured via the scaling function. To see this, recall that with Gaussian measurements, the scaling function takes the form $\varphi_t = \sqrt{\frac{2}{\pi}}(1-p)+\sqrt{\frac{2}{\pi}}p\mathbb{E}\left[e^{{-s_i^2}/({2\|U_tU_t-X^*\|_F})}\right]$. It is easy to see that, for small values of $\|U_tU_t-X^*\|_F$ (or alternatively, large values of noise), the scaling function can be well-approximated as $\varphi_t \approx \sqrt{\frac{2}{\pi}}(1-p)$. This implies that SubGD automatically takes more aggressive steps with increasing corruption probability.
\end{remark} 

Combining Theorem~\ref{convergence-theorem-noisy} and Theorem~\ref{thm_G_noisy} leads to an end-to-end sample complexity guarantee for SubGD with Gaussian measurements.

\begin{corollary}
Assume that the measurement matrices $\{A_i\}_{i=1}^m$ defining the linear operator $\mathcal{A}(\cdot)$ are symmetric with i.i.d. standard Gaussian entries, and that the noise vector $\mathbf{s}$ satisfies Assumption~\ref{assump_noise}. Moreover, suppose that  $\alpha$, $\eta_t$, and $T$ are chosen according to Theorem~\ref{convergence-theorem-noisy}. Then, SubGD satisfies the error bound~\eqref{eq_error} with an overwhelming probability, provided that $m\gtrsim \frac{dr'\log\left(\frac{1}{(1-p)\delta}\right)}{\delta^4(1-p)^4}$.
\end{corollary}

For any fixed $\delta$, the above corollary implies that SubGD converges to the vicinity of the true rank-1 solution, provided that the number of measurements scale as $dr'/(1-p)^4$ (modulo $\log$ factors). Note that our result holds for arbitrarily large values of $p$, provided that the number of measurements scale accordingly. This improves upon the existing result~\cite{li2020nonconvex}, which requires $p<1/2$ to guarantee the convergence of SubGD in the exact regime ($r'=r^*$).

\paragraph{Sketch of the proof for Theorem~\ref{convergence-theorem-noisy}.} Suppose that $X^*=u^*{u^*}^\top$ for $u^*\in\mathbb{R}^{d\times 1}$. Moreover, without loss of generality, we assume that $\|u^*\|=1$. Inspired by~\cite{li2018algorithmic}, we decompose the solution $U_t$ as
\begin{equation}\label{eq_decomposition}
    U_t = u^*{u^*}^\top U_t + \left(1-u^*{u^*}^\top\right)U_t := u^*r_t^\top + E_t,
\end{equation}
where $r_t = U_t^\top u^*$ is called \textit{signal term}, and $E_t = (1-u^*{u^*}^\top)U_t$ is referred to as \textit{error term}, which is the projection of $U_t$ onto the orthogonal complement of the subspace spanned by $u^*$. Evidently, we have $U_t U_t^{\top}=X^{\star}$ if and only if $\|r_t\|=1$ and $\|E_t\|_F = 0$. More generally, our next lemma shows that the error $\|U_tU_t-X^*\|_F$ can be controlled in terms of $\|E_t\|_F$ and $\|r_t\|$.
\begin{lemma}
    \label{decomposition-error-signal}
    The following inequality holds:
    \begin{equation}
        \norm{U_tU_t^{\top}\!-\!X^{\star}}_F^2\!\leq\!\left(1\!-\!\norm{r_t}^2\right)^2\!\!+2\norm{E_t}^2\!\norm{r_t}^2+\norm{E_t}_F^4.
    \end{equation}
\end{lemma}
Based on Lemma~\ref{decomposition-error-signal}, we provide a high-level idea of our proof technique:
\begin{enumerate}
    \item (Spectral Initialization) It is shown in Lemma~\ref{lem::spectral-init} that the proposed initialization scheme (see Algorithm~\ref{alg::spectral-initialization}) results in $\norm{r_0}=\alpha (1\pm O(\sqrt{\delta}))$ and $\norm{E_0}=  O(\alpha\sqrt{\delta})$. Therefore, the signal term dominates the error term at the beginning.
    \item It is shown in Lemma~\ref{signal-dynamics-noisy} that the signal term $\|r_t\|^2$ approaches $1$ at a geometric rate. Therefore, $1-\norm{r_t}^2$ converges to zero at a geometric rate.
    \item It is proven in Lemma~\ref{error-dynamics-noisy} that the error term $\|E_t\|_F$ grows at most {sublinearly}, and its growth rate is significantly slower than that of the signal term.
    \item  This discrepancy in the growth rates of the signal and error terms ensures that after a certain number of iterations $T$, the signal term $\|r_t\|$ is sufficiently close to 1, while the error term $\|E_t\|_F$ remains small. Combined with Lemma~\ref{decomposition-error-signal}, this establishes the convergence of SubGD with early stopping of the algorithm. 
\end{enumerate}

In particular, our main proof is based on the following three key lemmas, the proofs of which can be found in Appendix D of the supplementary material.
\begin{lemma}[Spectral Initialization]
    \label{lem::spectral-init}
    Suppose that $U_0=B_0$ is chosen by Algorithm~\ref{alg::spectral-initialization}. Then under the conditions of Theorem~\ref{convergence-theorem-noisy} , we have
\begin{equation}
    \norm{r_0}=\alpha\sqrt{\varphi_0} (1\pm O(\sqrt{\delta})), \quad \norm{E_0}=  O(\alpha\sqrt{\varphi_0}\sqrt{\delta}), \quad \norm{E_0}_F=O(\alpha\sqrt{\varphi_0}\sqrt[4]{r'}\sqrt{\delta}),
\end{equation}
where $\varphi_0=\varphi(U_0U_0^{\top}/\alpha^2-X^{\star})\in [\sqrt{2/\pi}(1-p),\sqrt{2/\pi}]$ is the initial scaling factor.
\end{lemma}
Given with this lemma, we next show that the signal term grows much faster than the error term.
\begin{lemma}[Signal Dynamics]
    \label{signal-dynamics-noisy}
    Assume that the measurements satisfy the sign-RIP with parameters $\left(\min\{r'+1,d\},\delta\right)$ and a strictly positive and uniformly bounded scaling function $\varphi(X)$. Moreover,
    suppose that $\norm{E_t}_F\leq 1, \norm{r_t}\leq 2, \delta\leq \frac{1}{2}$, and the step size $\eta_t$ is chosen as~\eqref{eq_stepsize2}. Then, we have
    \begin{equation}
        \begin{aligned}
            \norm{r_{t+1}-\left(1+\frac{\eta_0\rho^t (1-\norm{r_t}^2)}{\norm{U_tU_t^{\top}-X^{\star}}_F}\right)r_t}&\leq 2\delta\eta_0 \rho^t(\norm{E_t}+\norm{r_t})+\frac{2\eta_0\rho^t }{\norm{U_tU_t^{\top}-X^{\star}}_F}\norm{E_t}^2\norm{r_t}\\&+\frac{2\delta\eta_0\rho^t}{\norm{U_tU_t^{\top}-X^{\star}}_F}(1-\norm{r_t}^2)\norm{r_t}.
        \end{aligned}
    \end{equation}
\end{lemma}
The above lemma shows that, when $\|r_t\|$ and $t$ are small, the signal term grows geometrically fast. On the other hand, the growth rate of the error is sublinear, as shown in the following lemma.

\begin{lemma}[Error Dynamics]
    \label{error-dynamics-noisy}
    Suppose that the conditions of Proposition~\ref{signal-dynamics-noisy} are satisfied and $\eta_0\lesssim\delta\lesssim\norm{U_tU_t^\top-X^*}_F$. Then, we have
        \begin{align}
        &\norm{E_{t+1}}_F\leq \norm{E_t}_F+10\delta\eta_0\rho^t.
        \end{align}
\end{lemma}

Combining the aforementioned lemmas, we establish the global convergence of SubGD for robust rank-1 matrix recovery.$\hfill\square$



\section{Conclusion}

Existing techniques for analyzing low-rank matrix recovery presume and rely on different variants of restricted isometry property (RIP). However, these notions fail in the robust settings, where a number of measurements are grossly corrupted with noise. In this work, we propose a robust restricted isometry property, called Sign-RIP, that addresses this fundamental issue. Based on Sign-RIP, we paint a full picture for the landscape of robust rank-1 matrix recovery problem, both in the exact and over-parameterized regimes. In the exact setting, we show that all the critical points of the robust matrix recovery are close to the true solution, or have small norm. In the over-parameterized regime, we show that a simple subgradient method converges to the ground truth.

Although our results on robust matrix recovery is restricted to rank-1 case, the proposed framework is general, and it paves the way towards a better understanding of the problem in more general settings. In particular, our developed guarantees for sign-RIP hold for the general rank-$r$ matrices, and hence, can be potentially used to study the global landscape of more general robust matrix recovery problems in both exact and over-parameterized regimes. 

\newpage
\bibliography{reference}
\medskip

\appendix

\section*{\centering \textsc{\huge Appendix}}
\section{Numerical Experiments}
In this section, we provide extensive numerical experiments to verify our theoretical guarantees, and to shed light on possible future directions. 

All simulations are run on a desktop computer with an Intel Core i9 3.50 GHz CPU and 128GB RAM. The reported results are for an implementation in Python.
\subsection{Relationship between dimension and measurement number}

In this experiment, we analyze the relationship between the number of measurements $m$ and dimension $d$. Our theoretical result suggests that $m\gtrsim dr^{\prime}$ is enough to ensure the convergence of SubGD. To empirically verify this, we change $d$ from $10$ to $100$ and set $r^{\prime}=d$. Moreover, we set the corruption probability to $p=0.1$. Moreover, each element of the noise is generated according to a standard Gaussian distribution. The step sizes are selected as $\eta_t=\eta_0\rho^t$, where $\eta_0=0.4$ and $\rho=0.98$. For each group of parameters, we run $5$ independent trials and plot the average log-residual for the last iteration in Figure~\ref{figure::dependence-m-d}. It can be seen that, in order to ensure the same value for the error, the number of measurements should grow almost linearly with the dimension, which is in line with our theoretical result.
\begin{figure}[ht]
\begin{center}
\centerline{\includegraphics[width=10cm]{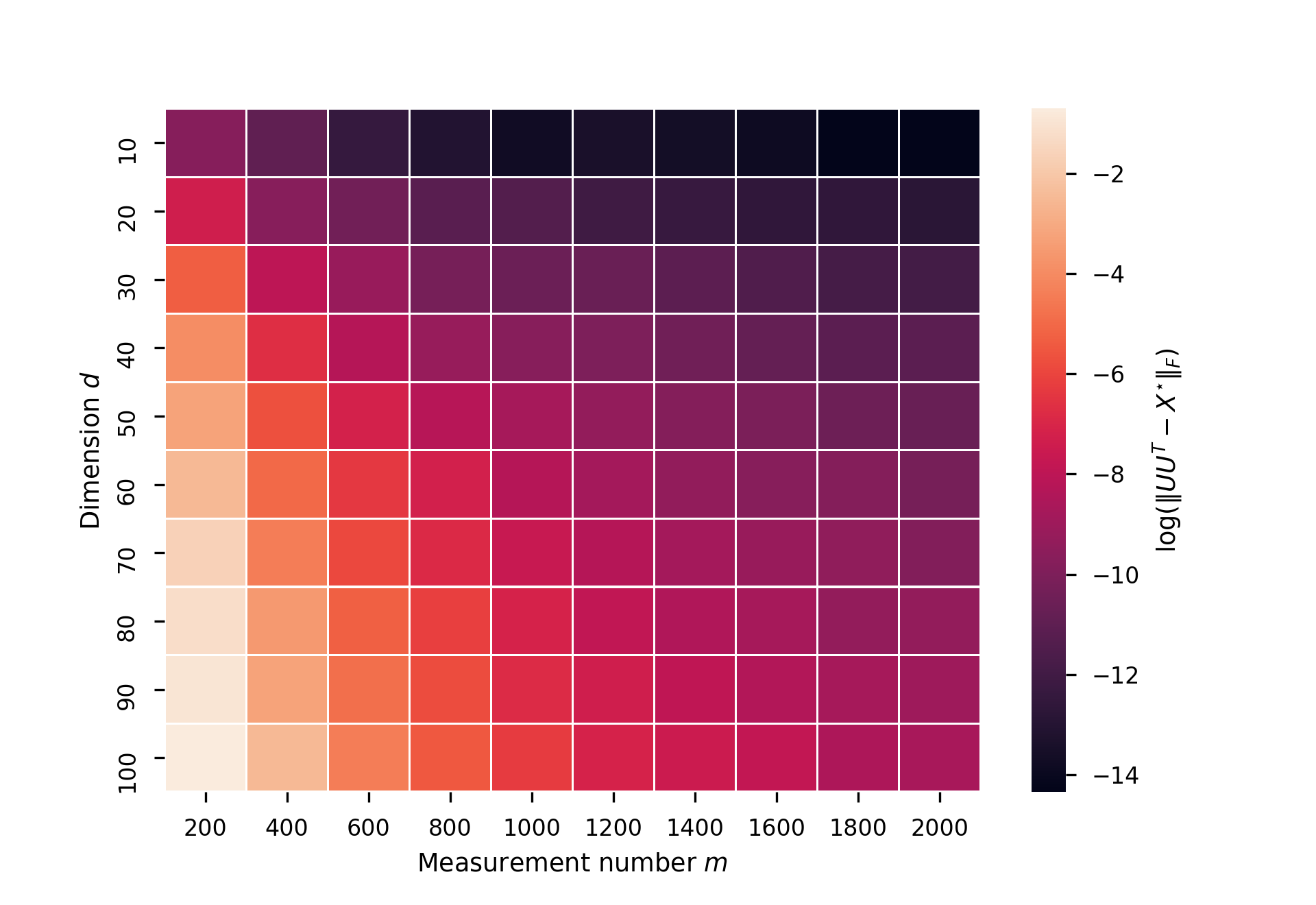}}
\caption{The error with respect to the number of measurements and dimension}
\label{figure::dependence-m-d}
\end{center}
\end{figure}

\subsection{Effect of Different Noise Magnitudes}
\label{sec::influence-of-different-noise-magnitude}
 In this experiment, we certify the robustness of SubGD against large noise values.
Our theoretical result suggests that the convergence of SubGD is independent of the noise magnitude.
To verify this, we set the dimension and the number of measurements to $d=50$ and $m=500$, respectively. Moreover, we set the corruption probability to $p=0.1$, and select each element of the noise according to a Gaussian distribution $s_i\sim\N(0,\sigma^2)$ with varying variance $\sigma^2$.
Finally, we set the step size to $\eta_t=\eta_0\rho^t$, where $\eta_0=0.25$, and $\rho=0.99$. Based on Figure~\ref{figure::dependence-noise-magnitude}, it can be seen that increasing variance slightly deteriorates the error. However, beyond a certain threshold, increasing variance does not have any effect on the error.
\begin{figure}[ht]
\begin{center}
\centerline{\includegraphics[width=10cm]{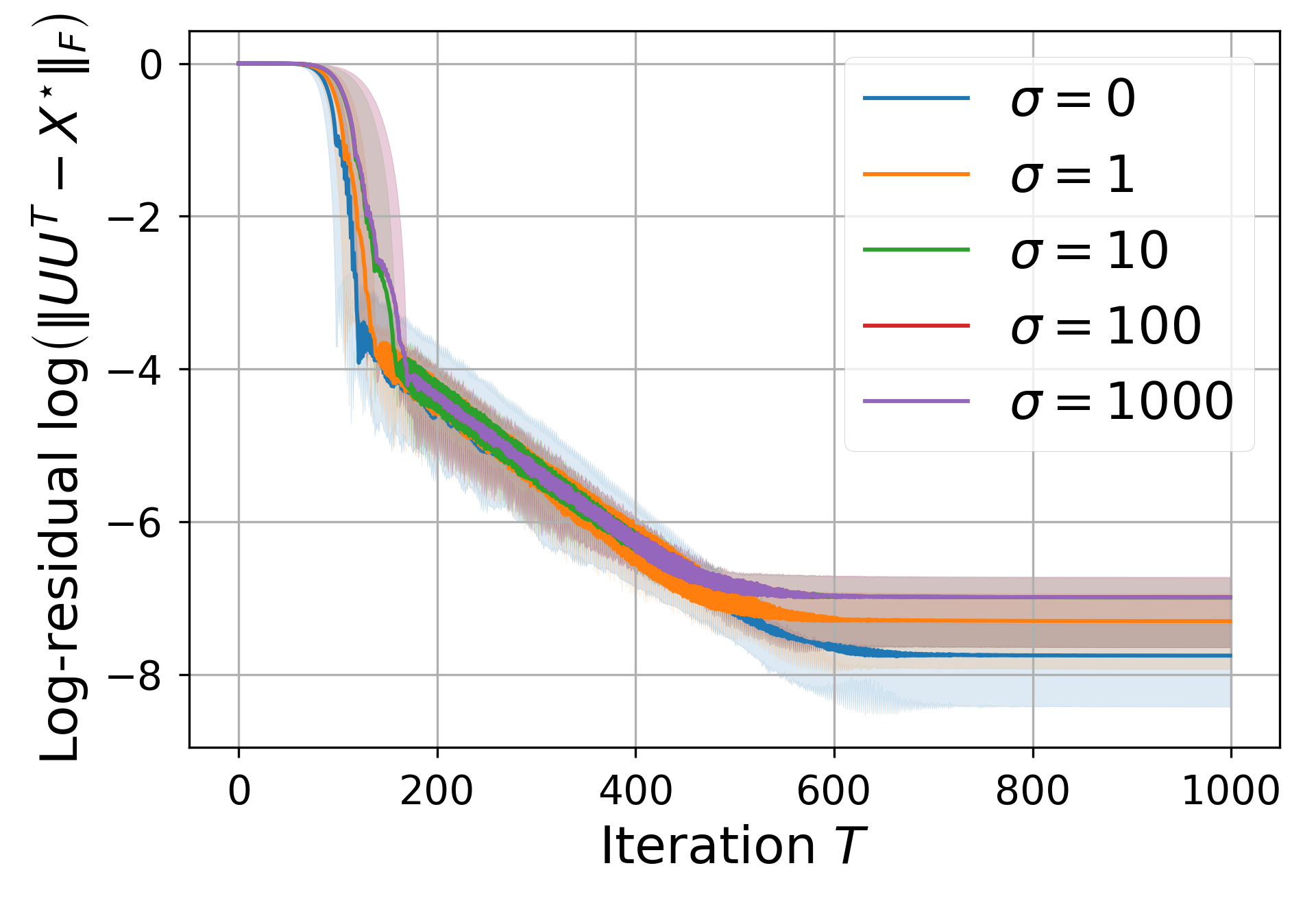}}
\caption{Effect of noise variance.}
\label{figure::dependence-noise-magnitude}
\end{center}
\end{figure}
\subsection{Effect of Different Types of Noise}
In this experiment, we study the effect of different types of noise on the performance of SubGD. In particular, we choose five different types of distribution for the noise: Gaussian, uniform, Laplace, Cauchy, and Rademacher. The experiments are designed under the same settings as Subsection~\ref{sec::influence-of-different-noise-magnitude}. Moreover, for all types of noise (except for the Cauchy distribution), we set the variance to $100$. As can be seen in Figure~\ref{figure::dependence-noise-type}, SubGD is insensitive to the particular choice of noise.
\begin{figure}[ht]
\begin{center}
\centerline{\includegraphics[width=10cm]{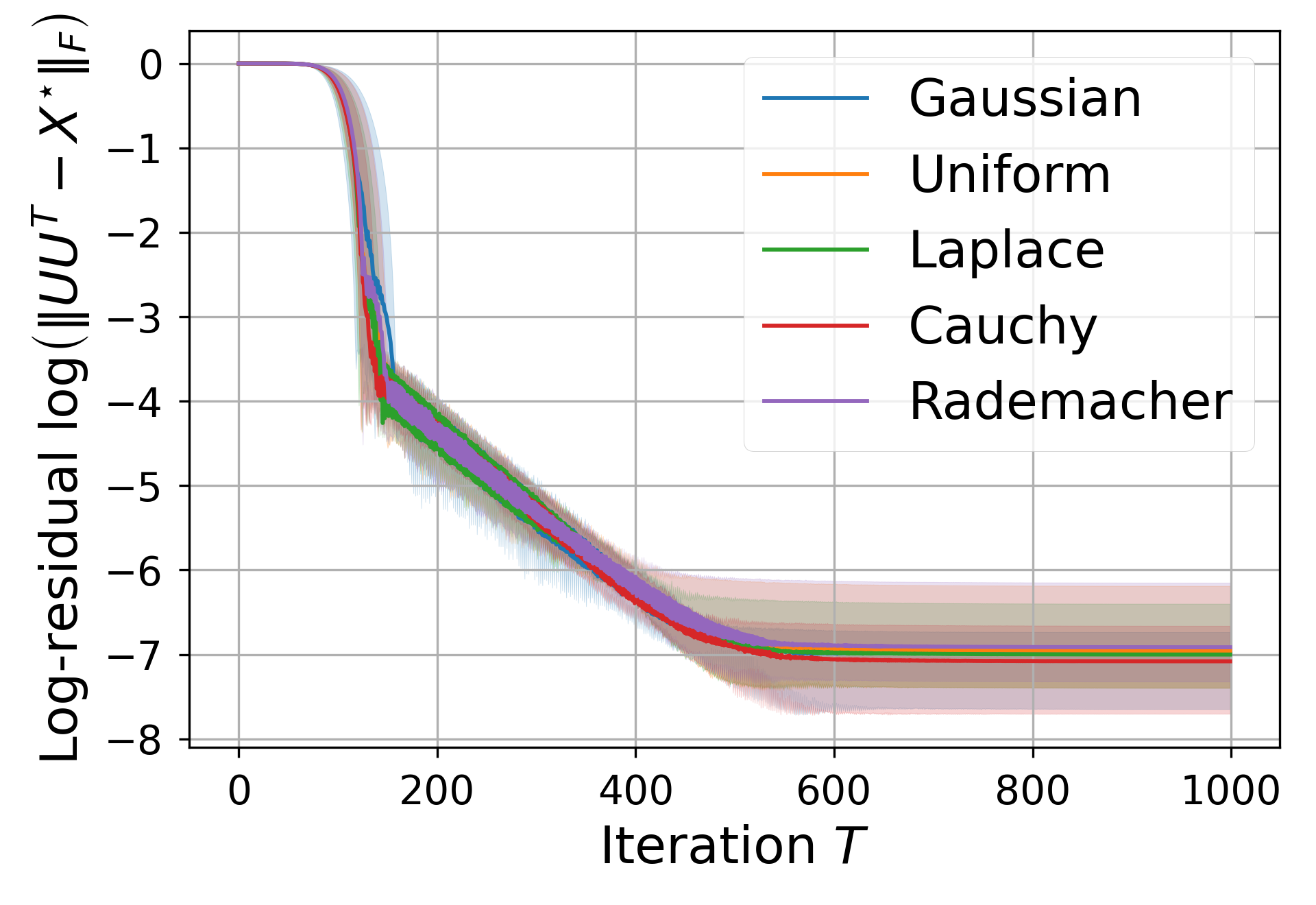}}
\caption{Effect of different types of noise.}
\label{figure::dependence-noise-type}
\end{center}
\end{figure}

\subsection{Effect of different step size regimes}
Finally, we explore the effect of different step sizes in both noiseless and noisy case under the same settings as Section~\ref{sec::influence-of-different-noise-magnitude}. In the noiseless case, we compare four different types of step sizes: 1) $\eta_t=\eta_0\rho^{t}$ with $\eta_0=0.25, \rho=0.99$; 2) $\eta_t = \frac{\eta_0}{t}$ with $\eta_0=2.0$; 3) $\eta_t=\frac{\eta_0}{\sqrt{t}}$ with $\eta_0=0.3$; and 4) our proposed choice $\eta_t=\eta_0 \frac{1}{m}\norm{\mathbf{y}-\mathcal{A}(U_tU_t^{\top})}_1$ where $\eta_0=0.25$ (see our discussion in Section 5 and Appendix D). As can be seen in Figure~\ref{fig::noiseless}, SubGD converges to the true solution with all of the aforementioned step sizes. However, our proposed step size leads to the fastest convergence rate. In the noisy case, we compare the performance of five different step sizes: 1) $\eta_t=\eta_0\rho^{t}$ with $\eta_0=0.45, \rho=0.98$; 2) $\eta_t = \frac{\eta_0}{t}$ with $\eta_0=2.0$; 3) $\eta_t=\frac{\eta_0}{\sqrt{t}}$ with $\eta_0=0.3$; 4) $\eta_t=\eta_0 \frac{1}{m}\norm{\mathbf{y}-\mathcal{A}(U_tU_t^{\top})}_1$ with $\eta_0=0.25$; 5) our proposed choice $\eta_t=\frac{\eta_0}{\norm{D_t}_F}\rho^t$, where $D_t\in\mathcal{M}(U_tU_t-X^*)$, $\eta_0=0.4$, and $\rho=0.99$. From Figure~\ref{fig::noisy}, it is evident that the step size $\eta_t=\eta_0 \frac{1}{m}\norm{\mathbf{y}-\mathcal{A}(U_tU_t^{\top})}_1$, which was the best choice in the noiseless case, does not result in the convergence of SubGD to the true solution in the noisy case. As mentioned before, this is due to the sensitivity of $\frac{1}{m}\norm{\mathbf{y}-\mathcal{A}(U_tU_t^{\top})}_1$ to outliers. Moreover, our proposed step size outperforms its vanilla counterpart. Finally, the polynomially decaying step size $\eta_t\propto\frac{1}{t}$ performs slightly better than our proposed step size. Motivated by this interesting observation, we will study the performance of SubGD with polynomially decaying step sizes in the future.

\begin{figure*}
\vskip 0.2in
\begin{center}
\subfloat[noiseless]{
{\includegraphics[width=7cm]{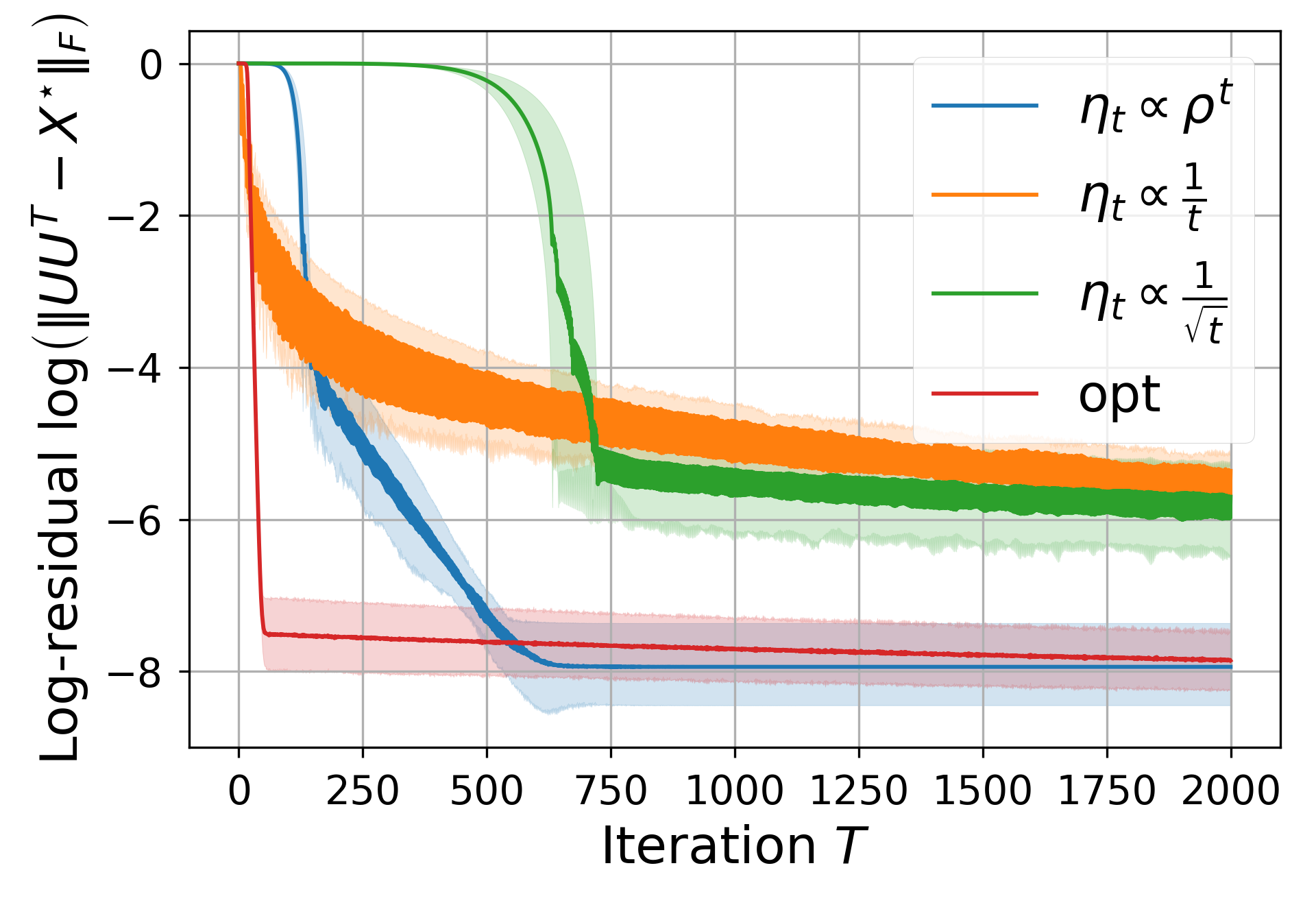}}\label{fig::noiseless}}
\subfloat[noisy]{
{\includegraphics[width=7cm]{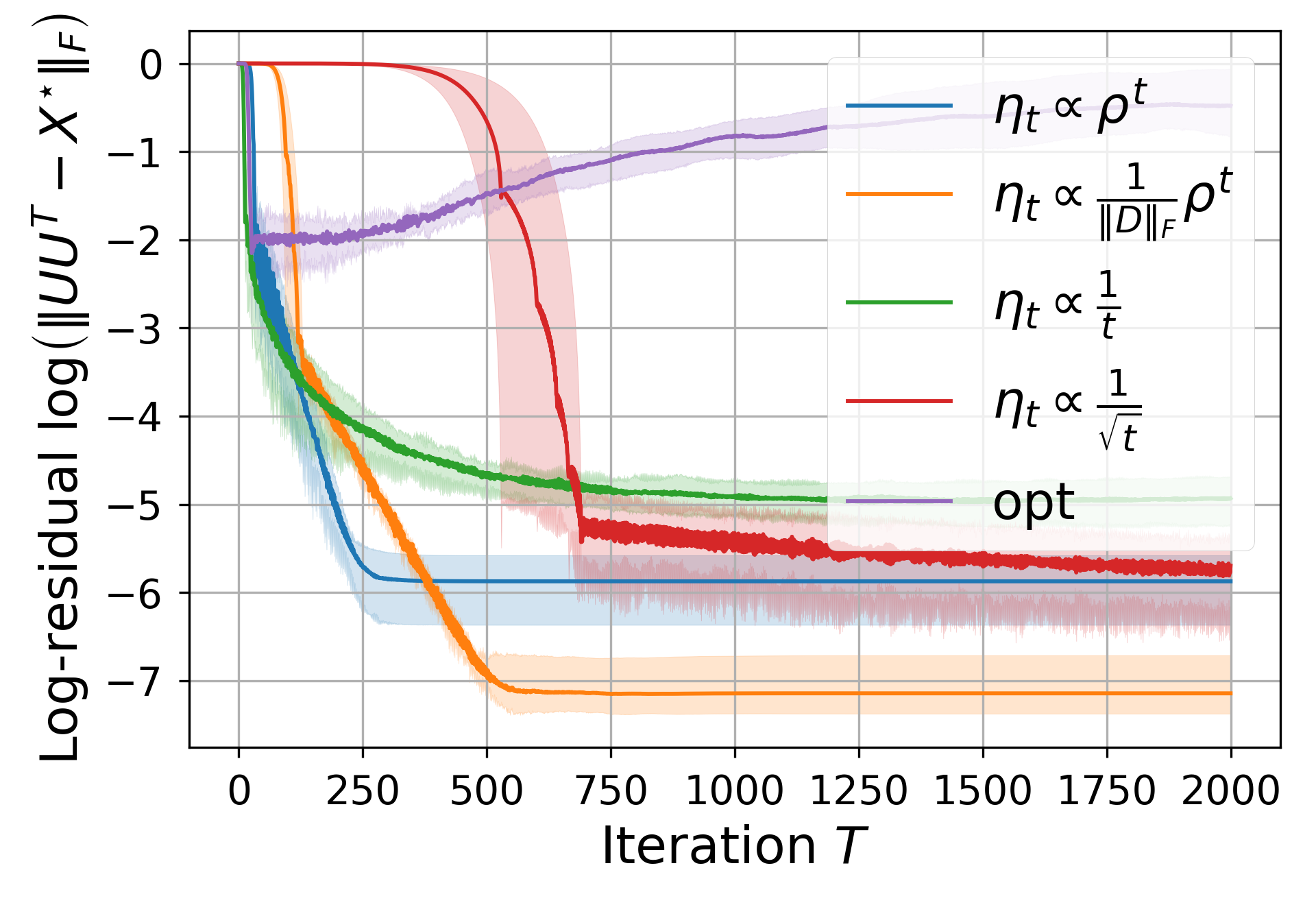}}\label{fig::noisy}}
\end{center}
\caption{Effect of different step size regimes.}
\end{figure*}

\section{Proofs for Sign-RIP}
In this section, we provide the proofs for Theorem~\ref{thm_G_noisy}. As a first step, we start with the noiseless case, and show that a weaker version of Theorem~\ref{thm_G_noisy} can be obtained directly from the so-called $\ell_1/\ell_2$-RIP condition.
\begin{lemma}[$\ell_1/\ell_2$-RIP, Proposition 1 in \citep{li2020nonconvex}]
\label{uniform-concentration-objective}
    Let $r\geq 1$ be given, suppose measurements $\{A_i\}_{i=1}^m$ have i.i.d. standard Gaussian entries with $m\gtrsim dr$. Then for any $0<\delta<\sqrt{\frac{2}{\pi}}$, there exists a universal constant $c>0$, such that with probability of at least $1-e^{-cm\delta^2}$, we have
    \begin{equation}
        \sup_{X\in \mathbb{S}_r}\left|\frac{1}{m}\sum_{i=1}^{m}|\inner{A_i}{X}|-\sqrt{\frac{2}{\pi}}\norm{X}_F\right|\leq \sqrt{\frac{2}{\pi}}\delta.
    \end{equation}
\end{lemma}
Based on the above $\ell_1/\ell_2$-RIP condition, we proceed to prove a weaker version of Theorem~\ref{thm_G_noisy}.
\begin{proposition}
Assume that the measurement matrices $\{A_i\}_{i=1}^m$ have i.i.d. standard Gaussian entries, and that $\mathbf{s} = 0$. Then, the sign-RIP condition holds with parameters $(r,\delta), \delta\leq \sqrt{\frac{2}{\pi}}$ and a constant scaling function $\varphi(X) = \sqrt{\frac{2}{\pi}}$ with probability of at least $1-Ce^{-cm\delta^4}$, provided that $m\gtrsim {d^2}$.
\end{proposition}

\begin{proof}
    Without loss of generality, we assume that $\norm{X}_F=1$. For any given $0<\delta\leq \sqrt{\frac{2}{\pi}}$ and any $D(X)\in \frac{1}{m}\sum_{i=1}^{m}\sign(\inner{A_i}{X})A_i$, we have
    \begin{equation}
        \begin{aligned}
            &\sup_{X\in \mathbb{S}_r}\norm{D(X)-\sqrt{\frac{2}{\pi}}X}_F^2\\
            \stackrel{(a)}{=}&\sup_{X\in \mathbb{S}_r}\norm{D(X)}_F^2 -\sqrt{\frac{8}{\pi}}\frac{1}{m}\sum_{i=1}^{m}|\inner{A_i}{X}|+\frac{2}{\pi}\\
            \leq&\sup_{X\in \mathbb{S}_r}\norm{D(X)}_F^2-\sqrt{\frac{8}{\pi}}\inf_{X\in \mathbb{S}_r}\frac{1}{m}\sum_{i=1}^{m}|\inner{A_i}{X}|+\frac{2}{\pi}\\
            \stackrel{(b)}{\leq}&\sup_{X\in \mathbb{S}_r}\norm{D(X)}_F^2+\sqrt{\frac{8}{\pi}}\sqrt{\frac{2}{\pi}}\delta -\sqrt{\frac{8}{\pi}}\sqrt{\frac{2}{\pi}}+ \frac{2}{\pi}\\
            =&\sup_{X\in \mathbb{S}_r}\norm{D(X)}_F^2+\frac{4}{\pi}\delta -\frac{2}{\pi}
        \end{aligned}
    \end{equation}
    with probability of at least $1-Ce^{-cm\delta^2}$. Here, (a) follows from $\inner{D(X)}{X}=\frac{1}{m}\sum_{i=1}^{m}|\inner{A_i}{X}|$, and (b) uses $\ell_1/\ell_2$-RIP condition from Lemma \ref{uniform-concentration-objective}. Then, recall that for an arbitrary $M\in\R^{d\times d}$, we have
    \begin{equation}
        \norm{M}_F=\sup_{Y\in\mathbb{S}}\inner{M}{Y}.
    \end{equation}
    This implies
    \begin{equation}
        \begin{aligned}
            \sup_{X\in \mathbb{S}_r}\norm{D(X)}_F^2&\leq \sup_{X, Y\in\mathbb{S}}\left(\inner{D(X)}{Y}\right)^2\\
            &\stackrel{(c)}{=}\sup_{Y\in \mathbb{S}}\left(\frac{1}{m}\sum_{i=1}^{m} |\inner{A_i}{Y}|\right)^2\\
            &\stackrel{(d)}{\leq} \frac{2}{\pi}(1+\delta)^2\\
            &\stackrel{(e)}{\leq} \frac{2}{\pi}+\frac{6}{\pi}\delta
        \end{aligned}
    \end{equation}
    with high probability $1-Ce^{-cm\delta^2}$. Here, (c) uses the fact that for a fixed $Y$, the supremum over $X$ is taken exactly at $X=Y$, (d) uses the $\ell_1/\ell_2$-RIP condition, and (e) uses the assumption $\delta\leq 1$.

    Combining these two parts, we obtain
    \begin{equation}
        \sup_{X\in \mathbb{S}}\norm{D(X)-\sqrt{\frac{2}{\pi}}X}_F^2\leq \frac{10}{\pi}\delta
    \end{equation}
    with probability of at least $1-Ce^{-c^{\prime}m\delta^2}$. Therefore, upon choosing $\delta^{\prime}=\sqrt{5\delta}$, we obtain \begin{equation}
        \sup_{X\in\mathbb{S}}\norm{D(X)-\sqrt{\frac{2}{\pi}}X}_F=\sup_{X,Y\in\mathbb{S}} \inner{D(X)-\sqrt{\frac{2}{\pi}}X}{Y}\leq \sqrt{\frac{2}{\pi}}\delta^{\prime}
    \end{equation}
    with probability of at least $1-Ce^{-cm\delta^{\prime 4}}$.
\end{proof}
Despite its simplicity, the above analysis has two major drawbacks: (1) its sample complexity scales with $d^2$, as opposed to $dr$ in Theorem~\ref{thm_G_noisy}, (2) it is not clear how to extend this analysis to the noisy case. To address these issues and prove Theorem~\ref{thm_G_noisy}, we need a more in-depth analysis of the sign-RIP condition.
First, we provide an intermediate lemma.

\begin{lemma}\label{l_expect}
    Assume that the measurement matrices $\{A_i\}_{i=1}^m$ defining the linear operator $\mathcal{A}(\cdot)$ have symmetric Gaussian entries, and that the noise vector $\mathbf{s}$ satisfies Assumption~\ref{assump_noise}. Then, for every $D\in\mathcal{M}(X)$, we have
    \begin{equation}
        \mathbb{E}[D] = \sqrt{\frac{2}{\pi}}\left(1-p+p\mathbb{E}\left[e^{{-s_i^2}/({2\|X\|_F})}\right]\right)\frac{X}{\|X\|_F}
    \end{equation}
    where the expectation is taken with respect to both $\mathbf{s}$ and $\{A_i\}_{i=1}^m$.
\end{lemma}
\begin{proof}
    To prove this lemma, it is enough to show that for any $X,Y\in\R^{d\times d}$, we have
    \begin{equation}
        \E\left[\sign\left(s+\inner{A}{X}\right)\inner{A}{Y}\right]=\sqrt{\frac{2}{\pi}}\E\left[e^{-s^2/2\norm{X}_F^2}\right]\inner{\frac{X}{\norm{X}_F}}{Y}.
    \end{equation}
    Without loss of generality, we assume that $\norm{X}_F=\norm{Y}_F=1$ and both $X$ and $Y$ are symmetric. Moreover, for the symmetric Gaussian matrix $A_i$, its off-diagonal elements are from $\N(0,1/2)$, and its diagonal elements are from $\N(0,1)$. Now let us denote $u:=\inner{A}{X},v:=\inner{A}{Y},\rho:=\Cov(u,v)=\inner{X}{Y}$. Then
    \begin{equation}
        \begin{aligned}
            \E\left[\sign\left(s+\inner{A}{X}\right)\inner{A}{Y}\right]&=\E\left[\sign\left(s+u\right)v\right]\\
            &\stackrel{(a)}{=}\rho \E\left[\sign(u+s)u\right]\\
            &=\rho\E_{s}\left[ \int_{-s}^{\infty}u \frac{1}{\sqrt{2\pi}}e^{-u^2/2}du-\int_{-\infty}^{-s}u \frac{1}{\sqrt{2\pi}}e^{-u^2/2}du\right]\\
            &=\rho\E_{s}\left[ \int_{-s}^{\infty}u \frac{1}{\sqrt{2\pi}}e^{-u^2/2}du+\int_{s}^{\infty}u \frac{1}{\sqrt{2\pi}}e^{-u^2/2}du\right]\\
            &=2\rho \E_{s}\left[\int_{|s_i|}^{\infty} u \frac{1}{\sqrt{2\pi}}e^{-u^2/2}du\right]\\
            &=\sqrt{\frac{2}{\pi}}\inner{X}{Y}\E_s\left[\int_{|s|}^{\infty}d\left(-e^{-u^2/2}\right)\right]\\
            &=\sqrt{\frac{2}{\pi}}\inner{X}{Y}\E_s\left[e^{-s^2/2}\right].
        \end{aligned}
    \end{equation}
    Here (a) uses the fact that $v|u, s\sim \N(\rho u, 1-\rho^2)$ since $s\perp \!\!\! \perp u, v$. This together with the variational form of the Frobenius norm implies
    \begin{equation}
        \E\left[\sign(s+\inner{A}{X})A\right]=\sqrt{\frac{2}{\pi}}\E\left[e^{-s^2/2\norm{X}_F^2}\right]\frac{X}{\norm{X}_F},
    \end{equation}
    for any $X\in\R^{d\times d}$. On the other hand, it is easy to verify that $\E\left[\sign\left(\inner{A}{X}\right)A\right] = \sqrt{\frac{2}{\pi}}\frac{X}{\norm{X}_F}$. The proof is completed by noting that the size of the noisy measurements is equal to $pm$.
\end{proof}

\subsection{Proof of Theorem~\ref{thm_G_noisy}}

For the sake of simplicity, we assume that $pm$ is an integer. Moreover, we abuse the notation and use $\sign(\cdot)$ as a regular function taking an arbitrary value $\sign(0)\in[-1, 1]$. To prove Theorem~\ref{thm_G_noisy}, we first present an intermediate lemma, which holds for any fixed $Y, Y \in\mathbb{S}$.

\begin{lemma}\label{l_subG}
    There exists a universal constant $c$, for any $\delta>0$, we have
    \begin{equation}
        \mathbb{P}\left(\left|\frac{1}{m}\sum_{i=1}^{m}\sign(\inner{A_i}{X}+s_i)\inner{A_i}{Y}-\sqrt{\frac{2}{\pi}}\left(1-p+p\E\left[e^{-s_i^2/2}\right]\right)\inner{X}{Y}\right|\geq \delta\right)\leq 2e^{-cm\delta^2}.
    \end{equation}
    \label{lem:single-concentration}
\end{lemma}
\begin{proof}[Proof of Lemma \ref{lem:single-concentration}]
\begin{sloppypar}
    We first show that $\sign(\inner{A_i}{X}+s_i)\inner{A_i}{Y}$ is a sub-Gaussian random variable. First notice that $\inner{A_i}{Y}\sim \N(0,1)$ since $\norm{Y}_F=1$. 
    Moreover, notice that  $\norm{\sign(\inner{A_i}{X}+s_i)\inner{A_i}{Y}}_{\ell^{2k}}\leq \norm{\inner{A_i}{Y}}_{\ell^{2k}}$ for $\forall k\in \mathbb{N}_+$, where $\|M\|_{\ell^{2k}}$ is defined as $\left(\E\left[|M|^p\right]\right)^{1/p}$. Therefore, based on equivalent definition of sub-Gaussian random variables (see Definition~\ref{def-sub-gaussian}), we obtain that $\sign(\inner{A_i}{X}+s_i)\inner{A_i}{Y}$ is also $O(1)$-sub-Gaussian. Moreover, according to the proof of Lemma~\ref{l_expect}, we have $\E\left[\frac{1}{m}\sum_{i\in S}\sign(\inner{A_i}{X})\inner{A_i}{Y}\right]=\sqrt{\frac{2}{\pi}}(1-p)\inner{X}{Y}$ and $\E\left[\frac{1}{m}\sum_{i\in S}\sign(\inner{A_i}{X}+s_i)\inner{A_i}{Y}\right]=\sqrt{\frac{2}{\pi}}p\E\left[e^{-s_i^2/2}\right]\inner{X}{Y}$.
    This together with the standard concentration bound on sub-Gaussian random variables leads to
\end{sloppypar}
    \begin{equation}
        \mathbb{P}\left(\left|\frac{1}{m}\sum_{i\notin S}\sign(\inner{A_i}{X})\inner{A_i}{Y}-\sqrt{\frac{2}{\pi}}\left(1-p\right)\inner{X}{Y}\right|\geq \frac{1}{2}\delta\right)\leq 2e^{-cm\delta^2/(1-p)},
    \end{equation}
    \begin{equation}
        \mathbb{P}\left(\left|\frac{1}{m}\sum_{i\in S}\sign(\inner{A_i}{X}+s_i)\inner{A_i}{Y}-\sqrt{\frac{2}{\pi}}p\E\left[e^{-s_i^2/2}\right]\inner{X}{Y}\right|\geq \frac{1}{2}\delta\right)\leq 2e^{-cm\delta^2/p}.
    \end{equation}
    which implies
    \begin{equation}
        \begin{aligned}
             &\mathbb{P}\left(\left|\frac{1}{m}\sum_{i=1}^{m}\sign(\inner{A_i}{X}+s_i)\inner{A_i}{Y}-\sqrt{\frac{2}{\pi}}\left(1-p+p\E\left[e^{-s_i^2/2}\right]\right)\inner{X}{Y}\right|\geq \delta\right)\\\leq& 4e^{-cm\delta^2\min\left\{\frac{1}{p},\frac{1}{1-p}\right\}}\leq 2e^{-c^{\prime}m\delta^2}.
        \end{aligned}
    \end{equation}
\end{proof}

Consider an $\epsilon$-covering $\mathbb{S}_{\epsilon, r}\subseteq\mathbb{S}_r$ with a property that for every $X\in \mathbb{S}_r$, there exists $\bar{X}\in \mathbb{S}_{\epsilon, r}$ that satisfies $\|X-\bar{X}\|_F\leq\epsilon$. According to Lemma~\ref{covering}, there exists an $\epsilon$-covering that satisfies $|\mathbb{S}_{\epsilon,r}|\leq \left(\frac{9}{\epsilon}\right)^{(2d+1)r}$.
For any $\bar{X}\in \mathbb{S}_{\epsilon,r}$, define $B_r(\bar{X},\epsilon)=\{X\in \mathbb{S}_r:\norm{X-\bar{X}}_F\leq \epsilon\}$. Then, for any $\bar{X},\bar{Y}$ and $X,Y\in B_r(\bar{X},\epsilon)\times B_r(\bar{Y},\epsilon)$, we have
\begin{equation}
\label{2-epsilon}
    \left|\inner{X}{Y}-\inner{\bar{X}}{\bar{Y}}\right|\leq \left|\inner{X-\bar{X}}{\bar{Y}}\right|+\left|\inner{X}{Y-\bar{Y}}\right|\leq 2\epsilon.
\end{equation}
Based on the defined $\epsilon$-covering, one can write
\begin{equation}
    \begin{aligned}
        &\sup_{X,Y\in\mathbb{S}_r}\left|\frac{1}{m}\sum_{i=1}^{m}\sign(\inner{A_i}{X}+s_i)\inner{A_i}{Y}-\sqrt{\frac{2}{\pi}}\left(1-p+p\E\left[e^{-s_i^2/2}\right]\right)\inner{X}{Y}\right|\\
        =&\sup_{\bar{X},\bar{Y}\in \mathbb{S}_{\epsilon,r}}\sup_{\substack{X\in B_r(\bar{X},\epsilon)\\Y\in B_r(\bar{Y},\epsilon)}}\left|\frac{1}{m}\sum_{i=1}^{m}\sign(\inner{A_i}{X}+s_i)\inner{A_i}{Y}-\sqrt{\frac{2}{\pi}}\left(1-p+p\E\left[e^{-s_i^2/2}\right]\right)\inner{X}{Y}\right|\\
        \leq&\underbrace{\sup_{\bar{X},\bar{Y}\in \mathbb{S}_{\epsilon,r}}\left|\frac{1}{m}\sum_{i=1}^{m}\sign(\inner{A_i}{\bar{X}}+s_i)\inner{A_i}{\bar{Y}}-\sqrt{\frac{2}{\pi}}\left(1-p+p\E\left[e^{-s_i^2/2}\right]\right)\inner{\bar{X}}{\bar{Y}}\right|}_{{\rm (A)}}\\
        +&\underbrace{\sup_{\bar{X},\bar{Y}\in \mathbb{S}_{\epsilon,r}}\sup_{Y\in B_r(\bar{Y},\epsilon)}\left|\frac{1}{m}\sum_{i=1}^{m}\sign(\inner{A_i}{\bar{X}}+s_i)\inner{A_i}{Y}-\sign(\inner{A_i}{\bar{X}}+s_i)\inner{A_i}{\bar{Y}}\right|}_{{\rm(B)}}\\
        +&\underbrace{\sup_{\bar{X}\in \mathbb{S}_{\epsilon,r},Y\in\mathbb{S}_r}\sup_{\substack{X\in B(\bar{X},\epsilon)}}\left|\frac{1}{m}\sum_{i=1}^{m}\sign(\inner{A_i}{\bar{X}}+s_i)\inner{A_i}{Y}-\sign(\inner{A_i}{X}+s_i)\inner{A_i}{Y}\right|}_{{\rm(C)}}\\
        +&\underbrace{\sup_{\bar{X},\bar{Y}\in \mathbb{S}_{\epsilon,r}}\sup_{\substack{X\in B_r(\bar{X},\epsilon)\\Y\in B_r(\bar{Y},\epsilon)}}\sqrt{\frac{2}{\pi}}\left(1-p+p\E\left[e^{-s_i^2/2}\right]\right)\left|\inner{X}{Y}-\inner{\bar{X}}{\bar{Y}}\right|}_{\leq \sqrt{\frac{8}{\pi}}\epsilon \text{ by }  \eqref{2-epsilon}}.
    \end{aligned}
\end{equation}
We control the first three terms separately. Based on a union bound and Lemma \ref{lem:single-concentration}, we have
\begin{equation}
    {\rm(A)}\leq \delta_1 \quad \text{with probability of at least } 1-2\left|\mathbb{S}_{\epsilon,r}\right|^2e^{-cm\delta_1^2}.
\end{equation}

Moreover, one can write
\begin{equation}
    \begin{aligned}
        {\rm(B)}&\leq \sup_{\bar{Y}\in \mathbb{S}_{\epsilon,r}}\sup_{\substack{Y\in B_r(\bar{Y},\epsilon)}}\frac{1}{m}\sum_{i=1}^{m}|\inner{A_i}{Y-\bar{Y}}|\\
        &\stackrel{(a)}{\leq} \epsilon\sup_{Z\in \mathbb{S}_{2r}}\frac{1}{m}\sum_{i=1}^{m}|\inner{A_i}{Z}|\\
        &\leq \sqrt{\frac{2}{\pi}}\epsilon(1+\delta_2)
    \end{aligned}
\end{equation}
with probability of at least $1-Ce^{c_1dr\log\frac{1}{\delta_2}-c_2m\delta_2^2}$. Here we used $\ell_1/\ell_2$ RIP condition from Lemma \ref{uniform-concentration-objective}, and the fact for $X,Y$ with ranks at most $r$, we have $\rank(X-Y)\leq \rank(X)+\rank(Y)\leq 2r$.
Next, we provide an upper bound for (C). First by Cauchy-Schwartz inequality, we have
\begin{equation}
    {\rm(C)}\leq \sup_{\bar{X}\in \mathbb{S}_{\epsilon,r}}\sup_{\substack{X\in B_r(\bar{X},\epsilon)}}\left(\underbrace{\frac{1}{m}\sum_{i=1}^{m}\left(\sign(\inner{A_i}{\bar{X}}+s_i)-\sign(\inner{A_i}{X}+s_i)\right)^2}_{{\rm(C1)}}\right)^{\frac{1}{2}}\sup_{Y\in\mathbb{S}_r}\left(\frac{1}{m}\sum_{i=1}^{m}\inner{A_i}{Y}^2\right)^{\frac{1}{2}}.
\end{equation}

The second term in the above inequality can be readily controlled via $\ell_2$-RIP (see Definition~\ref{def::l2-RIP}):
\begin{equation}
    \sup_{Y\in\mathbb{S}_r}\frac{1}{m}\sum_{i=1}^{m}\inner{A_i}{Y}^2\leq 1 + \delta_3
\end{equation}
which holds with probability of at least $1-Ce^{c_1dr\log\frac{1}{\delta_3}-c_2m\delta_3^2}$ for any $0<\delta_3<1$. For the remaining part (C1), first note that if $|\inner{A_i}{X-\bar{X}}|\leq |\inner{A_i}{\bar{X}+s_i}|$, then $\sign(\inner{A_i}{\bar{X}}+s_i)=\sign(\inner{A_i}{X}+s_i)$. This leads to

\begin{equation}
    \begin{aligned}
        \sup_{\bar{X}\in \mathbb{S}_{\epsilon,r}}\sup_{\substack{X\in B_r(\bar{X},\epsilon)}} {\rm(C1)}&\leq \sup_{\bar{X},X}\frac{4}{m}\sum_{i=1}^{m}\mathbbm{1}\left(|\inner{A_i}{X-\bar{X}}|\geq |\inner{A_i}{\bar{X}}+s_i|\right)\\
        &\stackrel{(a)}{\leq}\sup_{X,\bar{X}}\frac{4}{m}\sum_{i=1}^{m}\mathbbm{1}\left(|\inner{A_i}{X-\bar{X}}|\geq t\right)+\mathbbm{1}\left(|\inner{A_i}{\bar{X}}+s_i|\leq t\right)\\
        &\leq \sup_{Z\in \epsilon \mathbb{S}_{2r}}\frac{4}{m}\sum_{i=1}^{m}\mathbbm{1}\left(|\inner{A_i}{Z}|\geq t\right) + \sup_{\bar{X}}\frac{4}{m}\sum_{i=1}^{m} \mathbbm{1}\left(|\inner{A_i}{\bar{X}}+s_i|\leq t\right)\\& \stackrel{(b)}{\leq} \sup_{Z\in \epsilon \mathbb{S}_{2r}}\frac{4}{m}\sum_{i=1}^{m}\mathbbm{1}\left(|\inner{A_i}{Z}|\geq t\right) + 4\E\left[\mathbbm{1}\left(|\inner{A_i}{\bar{X}}+s_i|\leq t\right)\right]+\delta_4\\
        &\stackrel{(c)}{\leq} \sup_{Z\in \epsilon \mathbb{S}_{2r}}\frac{4}{m}\sum_{i=1}^{m}\mathbbm{1}\left(|\inner{A_i}{Z}|\geq t\right) + 4t+\delta_4
    \end{aligned}
\end{equation}
which holds with probability of at least $1-C|\mathbb{S}_{\epsilon,r}|e^{-cm\delta_4^2}$, where $t>0$ is an arbitrary scalar. Here, in (a) we use a simple fact that for two arbitrary random variables $A,B$ and a scalar $t\in \R$, the event $\{A\geq B\}$ is included in $\{A\geq t\}\cup \{B\leq t\}$. Moreover, in (b) we use a union bound and Hoeffding's inequality. Finally, in (c) we use the anti-concentration inequality conditioned on $s_i$.
For the first term in the above inequality, we have the following lemma.

\begin{lemma}\label{l_telegrand}
    We have
    \begin{equation}
        \E\left[\sup_{Z\in \epsilon \mathbb{S}_{2r}}\frac{4}{m}\sum_{i=1}^{m}\mathbbm{1}\left(|\inner{A_i}{Z}|\geq t\right)-\mathbbm{P}\left(|\inner{A_i}{Z}|\geq t\right)\right]\lesssim e^{-t^2/4\epsilon^2}\sqrt{\frac{dr}{m}}\vee \frac{dr}{m},
    \end{equation}
    moreover, for fixed $0<\delta<1$, we have the following tail bound 
    \begin{equation}
        \mathbb{P}\left(\left|\sup_{Z\in \epsilon \mathbb{S}_{2r}}\frac{4}{m}\sum_{i=1}^{m}\mathbbm{1}\left(|\inner{A_i}{Z}|\geq t\right)-\E\left[\sup_{Z\in \epsilon \mathbb{S}_{2r}}\frac{4}{m}\sum_{i=1}^{m}\mathbbm{1}\left(|\inner{A_i}{Z}|\geq t\right)\right]\right|>\delta\right)\leq 2e^{-cm\delta^2}.
    \end{equation}
\end{lemma}

\begin{proof}
    The tail bound directly follows from Theorem 8.5 in \citep{sen2018gentle}. Here we only give the proof sketch for the expectation bound. To apply Theorem 8.7 in \citep{sen2018gentle}, we only need to upper bound
    \begin{equation}
        \sigma^2:=\sup_{Z\in \epsilon \mathbb{S}_{2r}}\Var(\mathbbm{1}(|\inner{A}{Z}|>t)).
    \end{equation}
    Note that
\begin{equation}
    \begin{aligned}
        \sigma^2 &\leq \sup_{Z\in \epsilon \mathbb{S}_{2r}}\Var(\mathbbm{1}(|\inner{A}{Z}|>t))\\
        &\leq \sup_{Z\in \epsilon \mathbb{S}_{2r}} \E[\mathbbm{1}(|\inner{A}{Z}|>t)]\\
        &\leq \sup_{W\in \mathbb{S}}\mathbb{P}(|\inner{A}{W}|>t/\epsilon)\\
        &{\leq} 2e^{-t^2/2\epsilon^2},
    \end{aligned}
\end{equation}
where in the last inequality, we used the tail bound for Gaussian random variables.
Therefore, by Theorem 8.7 in \citep{sen2018gentle}, we have
\begin{equation}
    \begin{aligned}
        \E\left[\sup_{Z\in \epsilon \mathbb{S}_{2r}}\frac{4}{m}\sum_{i=1}^{m}\mathbbm{1}\left(|\inner{A_i}{Z}|\geq t\right)-\mathbbm{P}\left(|\inner{A_i}{Z}|\geq t\right)\right]&\lesssim \sigma\sqrt{\frac{dr}{m}\log\frac{1}{\sigma}}\vee \frac{dr}{m}\log\frac{1}{\sigma}\\&\lesssim e^{-t^2/4\epsilon^2}\sqrt{\frac{dr}{m}}\vee \frac{dr}{m}.
    \end{aligned}
\end{equation}
This completes the proof.
\end{proof}

Based on Lemma~\ref{l_telegrand}, we have
\begin{equation}
    \sup_{\bar{X},X} {\rm(C1)}\lesssim e^{-t^2/4\epsilon^2}+ 4t + \delta_4+\delta_5
\end{equation}
with probability of at least $1-C|\mathbb{S}_{\epsilon,r}|e^{-cm\delta_4^2}-Ce^{-cm\delta_5^2}$ given $m\gtrsim dr$.
Combining all derived bounds, we have
\begin{equation}
    \begin{aligned}
        &\sup_{X,Y\in\mathbb{S}_r}\left|\frac{1}{m}\sum_{i=1}^{m}\sign(\inner{A_i}{X}+s_i)\inner{A_i}{Y}-\left(1-p+p\E\left[e^{-s_i^2/2}\right]\right)\inner{X}{Y}\right|\\
        \leq & \delta_1+C\epsilon (1+\delta_2)+\sqrt{1+\delta_3}\sqrt{e^{-t^2/4\epsilon^2}+ 4t + \delta_4+\delta_5}
    \end{aligned}
\end{equation}
with probability of at least $1-2\left|\mathbb{S}_{\epsilon,r}\right|^2e^{-cm\delta_1^2}-Ce^{-cm\delta_2^2}-Ce^{c_1dr\log\frac{1}{\delta_3}-c_2m\delta_3^2}-C\left|\mathbb{S}_{\epsilon,r}\right|e^{-cm\delta_4^2}-Ce^{-cm\delta_5^2}$. Upon choosing $\delta_2=\delta_3=\frac{1}{2}$, $\delta_4=\delta_5=\delta^2$, $\delta_1=\delta$, $\epsilon=\delta^3$, $t=8\delta^3\sqrt{\log\frac{1}{\delta}}$, and $m\gtrsim dr(\log\frac{1}{\delta}\vee 1)/\delta^4$, we have
\begin{equation}
    \sup_{X,Y\in\mathbb{S}}\left|\frac{1}{m}\sum_{i=1}^{m}\sign(\inner{A_i}{X}+s_i)\inner{A_i}{Y}-\sqrt{\frac{2}{\pi}}\left(1-p+p\E\left[e^{-s_i^2/2}\right]\right)\inner{X}{Y}\right|\lesssim \delta,
\end{equation}
with probability of at least $1-Ce^{-c m\delta^4}$. This leads to
\begin{equation}
    \sup_{X\in\mathbb{S}, D\in\mathcal{M}(X)}\norm{D-\sqrt{\frac{2}{\pi}}\left(1-p+p\E\left[e^{-s_i^2/2}\right]\right)X}_F\lesssim \delta
\end{equation}
Finally, note that
\begin{equation}
    \sqrt{\frac{2}{\pi}}\left(1-p+p\E\left[e^{-s_i^2/2}\right]\right)\geq \sqrt{\frac{2}{\pi}}(1-p).
\end{equation}

Therefore, we have
\begin{equation}
    \sup_{X\in\mathbb{S}, D\in\mathcal{M}(X)}\norm{D-\sqrt{\frac{2}{\pi}}\left(1-p+p\E\left[e^{-s_i^2/2}\right]\right)X}_F\lesssim \sqrt{\frac{2}{\pi}}\left(1-p+p\E\left[e^{-s_i^2/2}\right]\right)\delta,
\end{equation}
with probability of at least $1-Ce^{-cm\delta^4}$, given $m\gtrsim \frac{dr\left(\log(\frac{1}{(1-p)\delta})\vee 1\right)}{(1-p)^4\delta^4}$.$\hfill\square$


\section{Proof of Theorem~\ref{thm_benign}}

Define $\Delta = UU^\top-u^*{u^*}^\top$. The conditions $\|u^*\|\leq 1$ and $\|U\|\leq R$ imply that $\|\Delta\|_F\leq (1+R^2)$. Since $U$ is a critical point of~\eqref{eq_l1}, we have $0\in\partial f_{\ell_1}(U)$, or equivalently, $QU = 0$ for some $Q\in\mathcal{Q}(\Delta)$,
which in turn implies that
\begin{align}
    \left(Q-\varphi(\Delta)\frac{\Delta}{\|\Delta\|_F}\right)U = -\varphi(\Delta)\frac{\Delta U}{\|\Delta\|_F}.\nonumber
\end{align}
Invoking Sign-RIP on the left side, we have
\begin{align}
    {\|\Delta U\|}\leq \delta {\|\Delta\|}\|U\|\leq R(1+R^2)\delta.\label{eq_delta}
\end{align}
Consider the decomposition $U = \alpha u^* + \beta u^*_{\perp}$, where $\alpha = \langle u^*,U\rangle$, $u^*_{\perp} = \frac{U-\langle u^*,U\rangle u^*}{\|U-\langle u^*,U\rangle u^*\|}$, and $\beta = \|U-\langle u^*,U\rangle u^*\|$. Note that $\|u^*\| = \|u^*_\perp\| = 1$ and $\langle u^*,u^*_\perp\rangle = 0$. Substituting this decomposition in inequality~\eqref{eq_delta} leads to 
\begin{align}\label{eq_alpha_beta}
    \|\Delta U\|^2 &= \|(UU^\top-u^*{u^*}^\top)U\|^2\nonumber\\
    &= \|\alpha(\alpha^2+\beta^2-1)u^*+\beta(\alpha^2+\beta^2)u^*_\perp\|^2\nonumber\\
    &= \alpha^2(\alpha^2+\beta^2-1)^2+\beta^2(\alpha^2+\beta^2)^2\nonumber\\
    &\leq R^2(1+R^2)^2\delta^2
\end{align}

The above inequality implies that 
\begin{align}
    \alpha(\alpha^2+\beta^2-1)\leq R(1+R^2)\delta\qquad\quad \text{and}\qquad\quad \beta(\alpha^2+\beta^2)\leq R(1+R^2)\delta\nonumber
\end{align}
which in turn leads to
\begin{align}
    &\left\{\underbrace{|\alpha|\leq (R(1+R^2)\delta)^{1/3}}_{(A)}\quad \text{or}\quad \underbrace{|\alpha^2+\beta^2-1|\leq (R(1+R^2)\delta)^{2/3}}_{(B)}\right\}\nonumber\\
    &\hspace{4cm}\text{and}\label{eq_and_or}\\
    &\left\{\underbrace{|\beta|\leq (R(1+R^2)\delta)^{1/3}}_{(C)}\quad \text{or}\quad \underbrace{|\alpha^2+\beta^2|\leq (R(1+R^2)\delta)^{2/3}}_{(D)}\right\}\nonumber
\end{align}
Now, we consider different cases
\begin{itemize}
    \item[-] {\it Case 1:} Suppose that (A) and (C) holds. This implies that $$\alpha^2+\beta^2\leq 2(R(1+R^2)\delta)^{2/3}\implies \|U\|^2\leq 2(R(1+R^2)\delta)^{2/3}.$$
    \item[-] {\it Case 2:} Suppose that (A) and (D) holds. Similar to the previous case, we have $$\alpha^2+\beta^2\leq (R(1+R^2)\delta)^{2/3}\implies \|U\|^2\leq (R(1+R^2)\delta)^{2/3}.$$
    \item[-] {\it Case 3:} Suppose that (B) and (D) holds. This implies that
    $1-(R(1+R^2)\delta)^{2/3}\leq \alpha^2+\beta^2$ and $\alpha^2+\beta^2\leq (R(1+R^2)\delta)^{2/3}$. However, these inequalities are mutually exclusive if $\delta< \frac{1}{4\sqrt{2}R^3}$.
    \item[-]{\it Case 4:} Suppose that (B) and (C) holds. This implies that
    $$1-2(R(1+R^2)\delta)^{2/3}\leq \alpha^2\leq 1+(R(1+R^2)\delta)^{2/3}\implies |\alpha^2-1|\leq 2(R(1+R^2)\delta)^{2/3}$$
    which leads to
    $$\|UU^\top-u^*{u^*}^\top\|_F\leq \sqrt{5}(R(1+R^2)\delta)^{2/3}.$$
\end{itemize}
In summary, we have shown that for sufficiently small $\delta$, the critical point $\|UU^\top-u^*{u^*}^\top\|_F\leq \sqrt{5}(R(1+R^2)\delta)^{2/3}$ or $\|U\|^2\leq 2(R(1+R^2)\delta)^{2/3}$. Our next goal is to show that these bounds can be improved to $\|U\|^2\leq 2\sqrt{2}\delta$ and $\|UU^\top-u^*{u^*}^\top\|_F\leq 2\sqrt{2}\delta$.

Suppose that $U$ is close to $0$, i.e., $\|U\|^2\leq 2(R(1+R^2)\delta)^{2/3}$ (Cases 1 and 2). We use the following intermediate lemma:
\begin{lemma}\label{l_U0}
 Suppose that $\|U\|\leq C\delta^{\epsilon}$ for some $C\geq 1$ and $1\geq \epsilon\geq 1/3$. Moreover, suppose that Sign-RIP holds with $\delta\leq \frac{1}{4C^{1/\epsilon}}$. Then, we have $\|U\|\leq \sqrt{2}(C)^{1/3}\delta^{(\epsilon+1)/3}$.
\end{lemma}
\begin{proof}
Since $\|U\|\leq C\delta^{\epsilon}$, we have $\|\Delta U\|^2\leq 4C^2\delta^{2+2\epsilon}$, provided that $\delta\leq \frac{1}{(C)^{1/\epsilon}}$. A similar argument to~\eqref{eq_and_or} can be used to show that
\begin{align}
    &\left\{\underbrace{|\alpha|\leq (2C)^{1/3}\delta^{(1+\epsilon)/3}}_{(A)}\quad \text{or}\quad \underbrace{|\alpha^2+\beta^2-1|\leq (2C)^{2/3}\delta^{2(1+\epsilon)/3}}_{(B)}\right\}\nonumber\\
    &\hspace{4cm}\text{and}\nonumber\\
    &\left\{\underbrace{|\beta|\leq (2C)^{1/3}\delta^{(1+\epsilon)/3}}_{(C)}\quad \text{or}\quad \underbrace{|\alpha^2+\beta^2|\leq (2C)^{2/3}\delta^{2(1+\epsilon)/3}}_{(D)}\right\}\nonumber
\end{align}
It is easy to verify that (B) is infeasible due to the upper bound on $\delta$. Therefore, we have $\|U\|^2 = \alpha^2+\beta^2\leq 2(2C)^{2/3}\delta^{2(1+\epsilon)/3}$, which completes the proof.
\end{proof}
Upon choosing $C_0 = \sqrt{2}(R(1+R^2))^{1/3}$ and $\epsilon_0=1/3$, the repeated applications of Lemma~\ref{l_U0} implies that 
$$\|U\|\leq C_k\delta^{\epsilon_k}, \quad\text{where}\quad C_k = \sqrt{2}C_{k-1}^{1/3},\ \ \epsilon_k = \frac{1+\epsilon_{k-1}}{3}$$
for every $k=1,2,\dots,\infty$, provided that $\delta\leq \frac{1}{4C_k^{1/\epsilon_k}}$. It is easy to verify that $C_{\infty} = 2^{3/4}$ and $\epsilon_{\infty} = 1/2$. Moreover, the choice of $\delta\leq \frac{1}{32R(1+R^2)}$ is enough to guarantee $\delta\leq \frac{1}{4C_k^{1/\epsilon_k}}$. This in turn implies that $\|U\|^2\leq 2\sqrt{2}\delta$. 
A similar approach can be used to show that $\|UU^\top-u^*{u^*}^\top\|_F\leq 2\sqrt{2}\delta$.$\hfill\square$

Based on the proof of Theorem~\ref{thm_benign}, we prove Corollary~\ref{cor_benign}.

{\it Proof of Corollary~\ref{cor_benign}.} The first part of the corollary is directly implied by Theorem~\ref{thm_G_noisy} combined with Theorem~\ref{thm_benign}. To prove the second part, note that, due to Lemma~\ref{l_RIP}, Sign-RIP implies $\ell_1/\ell_2$-RIP. On the other hand,~\citet{li2020nonconvex} show that, with the choice of $p\leq \frac{1}{2}-\frac{\delta}{\sqrt{2/\pi}-\delta}$, the loss function is $\alpha$-weakly-sharp~\cite[Proposition 2]{li2020nonconvex} and $\tau$-weakly-convex~\cite[Proposition 3]{li2020nonconvex}, with appropriate choices of $\alpha$ and $\tau$ (for simplicity of presentation, we omit the exact definition of these parameters, and refer the reader to~\cite{li2020nonconvex}). This implies that the $f_{\ell_1}(U)$ does not have a critical point other than $U = u^*$ within a neighborhood of $u^*$ with radius $\frac{2\alpha}{\tau}$. On the other hand, Theorem~\ref{thm_benign} implies that, with the provided bounds on $\delta$ and $m$, all the critical points of that are not close to the origin must be within a neighborhood of $u^*$ with radius $\frac{2\alpha}{\tau}$. This implies that these critical points must coincide with true solution, i.e., $UU^\top = u^*{u^*}^\top$. $\hfill\square$

\section{Proofs for Convergence Analysis}

First, we consider the noiseless scenario with clean measurements. Our proof for the noisy case is built upon the developed results for the the noiseless scenario. In particular, different from the noisy setting, we choose the step-size $\eta_t = \frac{\pi}{2m}\norm{\mathsf{y}-\mathcal{A}(U_tU_t^{\top})}_1$. As will be shown later, this choice of $\eta_t$ will remain close to $\|UU^\top - {u^*}{u^*}^\top\|_F$ due to sign-RIP.
\section*{Proofs for the Noiseless Case}

\subsection{Proof of Lemma \ref{decomposition-error-signal}}
    Due to~\eqref{eq_decomposition}, one can write
    \begin{equation}
        \begin{aligned}
           &\norm{U_tU_t^{\top}-X^{\star}}_F^2\\
           =&\norm{(u^{\star}r_t^{\top}+E_t)(r_t {u^{\star}}^{\top}+E_t^{\top})-u^{\star}{u^{\star}}^{\top}}_F^2\\
           =&\norm{\underbrace{(\norm{r_t}^2-1)u^{\star}{u^{\star}}^{\top}}_{\rm (A)}+\underbrace{E_t E_t^{\top}}_{\rm (B)}+\underbrace{E_tr_tu^{\star\top}}_{\rm (C)}+\underbrace{u^{\star}r_t^{\top}E_t^{\top}}_{\rm (D)}}_F^2.
        \end{aligned}
    \end{equation}
    Now, note that $\norm{\rm A}_F^2=\left(1-\norm{r_t}^2\right)^2$, and 
    \begin{equation}
        \norm{\rm C}_F=\norm{\rm D}_F=\norm{E_t r_t}\norm{u^{\star}}=\norm{E_t r_t}.
    \end{equation}
    \begin{sloppypar}
    On the other hand, $u^{\star\top}E_t=u^{\star\top}(I-u^{\star}u^{\star\top})U_t=0$, and therefore, $\inner{\rm A}{\rm B}=0$. Similarly, $\inner{\rm A}{\rm C}=(\norm{r_t}^2-1)\tr{u^{\star}u^{\star\top}E_t r_tu^{\star\top}}=0$, and $\inner{\rm A}{\rm D}=0$ since $\tr{u^{\star}u^{\star\top}u^{\star}r_t^{\top}E_t^{\top}}=\tr{u^{\star}u^{\star\top}U_tU_t^{\top}(I-u^{\star}u^{\star\top})}=\tr{U_tU_t^{\top}(I-u^{\star}u^{\star\top})u^{\star}u^{\star\top}}=0$. Similarly, we have $\inner{\rm B}{\rm C}=0, \inner{\rm B}{\rm D}=0, \inner{\rm C}{\rm D}=0$. This leads to
    \end{sloppypar}
    \begin{equation}\label{eq_error_U}
        \begin{aligned}
            \norm{U_tU_t^{\top}-X^{\star}}_F^2&=\left(1-\norm{r_t}^2\right)^2+2\norm{E_tr_t}^2+\norm{E_tE_t^{\top}}_F^2\\
            &\leq \left(1-\norm{r_t}^2\right)^2+2\norm{E_t}^2\norm{r_t}^2+\norm{E_t}_F^4.
        \end{aligned}
    \end{equation} 
$\hfill\square$

\subsection{Error Dynamics}
\begin{proposition}[Error Dynamics]
    \label{error-dynamic}
    Assume that the measurements are noiseless and satisfy the sign-RIP with parameters $\left(\min\{r'+1,d\},\delta\right)$, and constant scaling function $\varphi(X) = \sqrt{\frac{2}{\pi}}$. Moreover, suppose that $\norm{E_t}_F\leq 1,\norm{r_t}\leq 2, \delta\leq \frac{1}{2}$, and the step size $\eta_t$ is chosen as~\eqref{eq_stepsize} with $\eta_0\leq \frac{2}{45}$. Then, we have
    \begin{align}
        &\norm{E_{t+1}}_F\leq \norm{E_t}_F+22\delta\eta_0.
        \label{eq.5.27}\\
        &\norm{E_{t+1}}\leq \norm{E_t}+15\delta\eta_0.
    \end{align}
\end{proposition}
\begin{proof}

For simplicity of notation, we define $\Delta_t = U_tU_t^\top-X^*$ throughout the proof.
First, we provide a useful fact, which will be widely used in our subsequent arguments.
\begin{fact}
    \label{claim-decomposition}
    Suppose $\norm{E_t}_F\leq 1, \norm{r_t}\leq 2$, then by Lemma \ref{decomposition-error-signal}, we have $\norm{\Delta_t}_F^2\leq 1+2\times 4+1=10$.
\end{fact}

We first prove the error dynamics under a general learning rate $\eta_t$. 
\begin{lemma}
    \label{general-error-dynamics}
    Suppose that $\norm{E_t}_F\leq 1,\norm{r_t}\leq 2$, $\delta\leq \frac{1}{2}$, then, the following inequalities hold
    \begin{equation}
        \norm{E_{t+1}}_F^2\leq \norm{E_t}_F^2+\sqrt{\frac{8}{\pi}}\eta_t\left(-\frac{\norm{E_tU_t^{\top}}_F^2}{\norm{\Delta_t}_F}+\delta \norm{E_tU_t^{\top}}_F\right)+\frac{20}{\pi}\eta_t^2\left(\delta^2+\norm{E_tU_t^{\top}}_F^2/\norm{\Delta_t}_F^2\right),
    \end{equation}
    \begin{equation}
        \norm{E_{t+1}}\leq \norm{I-\frac{\eta_t U_t^{\top}U_t}{\sqrt{\frac{\pi}{2}}\norm{\Delta_t}_F}}\cdot \norm{E_t}+\sqrt{\frac{2}{\pi}}\delta\eta_t (\norm{r_t}+\norm{E_t}).
    \end{equation}
\end{lemma}
\begin{proof}
    For simplicity, we denote $M_t\in \frac{1}{m}\sum_{i=1}^{m}\sign(\inner{A_i}{U_tU_t^{\top}-X^{\star}})A_i$, and $\bar{M}_t=\sqrt{\frac{2}{\pi}}\frac{\Delta_t}{\norm{\Delta_t}_F}$.
    It is easy to verify that
    \begin{equation}
        E_{t+1}=E_t-\eta_t(I-u^{\star}u^{\star\top})M_t U_t,
    \end{equation}
    Based on the above equation, one can write
    \begin{equation}\label{eq_error_frob}
        \norm{E_{t+1}}_F^2=\norm{E_t}_F^2-2\eta_t \inner{E_t}{(I-u^{\star}u^{\star\top})M_t U_t}+\eta_t^2 \norm{(I-u^{\star}u^{\star\top})M_t U_t}_F^2.
    \end{equation}
    Next, we will provide separate upper bounds for the second and third terms in the above equation. First, note that
    \begin{equation}\label{eq_middle}
        \begin{aligned}
            \inner{E_t}{M_t U_t}&=\inner{M_t}{E_tU_t^{\top}}\\
            &\stackrel{(a)}{\geq}\frac{1}{\sqrt{\frac{\pi}{2}}\norm{\Delta_t}_F}\inner{\Delta_t}{E_t U_t^{\top}}-\sqrt{\frac{2}{\pi}}\delta \norm{E_tU_t^{\top}}_F\\
            &\stackrel{(b)}{=}\sqrt{\frac{2}{\pi}}\frac{\norm{E_tU_t^{\top}}_F^2}{\norm{\Delta_t}_F}-\sqrt{\frac{2}{\pi}}\delta \norm{E_tU_t^{\top}}_F.
        \end{aligned}
    \end{equation}
    where we used sign-RIP condition in (a), and (b) follows from $\inner{U_tU_t^{\top}-X^{\star}}{E_tU_t^{\top}}=\inner{E_tU_t^{\top}}{E_tU_t^{\top}}=\norm{E_tU_t^{\top}}_F^2$.
    On the other hand, we have
    \begin{equation}\label{eq_middle2}
        \inner{E_t}{u^{\star}u^{\star\top}M_t U_t}=\tr{E_t^{\top}u^{\star}u^{\star\top}M_t U_t}=\tr{U_t^{\top}(I-u^{\star}u^{\star\top})u^{\star}u^{\star\top}M_t U_t}=0.
    \end{equation}
    Combining~\eqref{eq_middle} and~\eqref{eq_middle2} leads to 
    \begin{align}
        -2\eta_t \inner{E_t}{(I-u^{\star}u^{\star\top})M_t U_t}\leq -2\eta_t\left(\sqrt{\frac{2}{\pi}}\frac{\norm{E_tU_t^{\top}}_F^2}{\norm{\Delta_t}_F}-\sqrt{\frac{2}{\pi}}\delta \norm{E_tU_t^{\top}}_F\right)
    \end{align}
    Now, we provide an upper bound for the third term in~\eqref{eq_error_frob}. One can write
    \begin{equation}
        \begin{aligned}
            \norm{(I-u^{\star}u^{\star\top})M_t U_t}_F^2&\leq 2\norm{(I-u^{\star}u^{\star\top})(M_t-\bar{M}_t)U_t}^2_F+2\norm{(I-u^{\star}u^{\star\top})\bar{M}_tU_t}^2_F\\
            &\stackrel{(a)}{\leq}2\norm{(M_t-\bar{M}_t)U_t}^2_F+\frac{4}{\pi}\frac{\norm{E_tU_t^{\top}U_t}^2_F}{\norm{\Delta_t}_F^2}\\
            &\stackrel{(b)}{\leq}\frac{4}{\pi}\delta^2\norm{U_t}_F^2+\frac{4}{\pi}\norm{E_tU_t^{\top}}_F^2\norm{U_t}_F^2/\norm{\Delta_t}_F^2\\&\stackrel{(c)}{\leq} \frac{20}{\pi}\delta^2+\frac{20}{\pi}\norm{E_tU_t^{\top}}_F^2/\norm{\Delta_t}_F^2.
        \end{aligned}
    \end{equation}
    where we used the contraction of projection and $(I-u^{\star}u^{\star\top})(U_tU_t^{\top}-X^{\star})U_t=E_tU_t^{\top}U_t$ in (a). Moreover, (b) directly follows from the sign-RIP condition. Finally, we used the following fact in (c).

    \begin{fact}\label{claim-U}
        Assuming $\norm{E_t}_F\leq 1, \norm{r_t}\leq 2$, we have $\norm{U_t}_F^2=\norm{r_t}^2+\norm{E_t}_F^2\leq 1^2+2^2=5$.
    \end{fact}

    Finally, combining all the three terms, we finally have:
    \begin{equation}
        \norm{E_{t+1}}_F^2\leq \norm{E_t}_F^2+\sqrt{\frac{8}{\pi}}\eta_t\left(-\frac{\norm{E_tU_t^{\top}}_F^2}{\norm{\Delta_t}_F}+\delta \norm{E_tU_t^{\top}}_F\right)+\frac{20}{\pi}\eta_t^2\left(\delta^2+\frac{\norm{E_tU_t^{\top}}_F^2}{\norm{\Delta_t}_F^2}\right).
    \end{equation}

    Now, we turn to control the spectral norm. First,  notice that
    \begin{equation}
        \begin{aligned}
            &\norm{(I-u^{\star}u^{\star\top})M_tU_t-(I-u^{\star}u^{\star\top})\frac{\Delta_tU_t}{\sqrt{\frac{\pi}{2}}\norm{\Delta_t}_F}}\\
            \stackrel{(a)}{\leq} & \norm{M_tU_t-\frac{\Delta_tU_t}{\sqrt{\frac{\pi}{2}}\norm{\Delta_t}_F}}\\
            \leq & \norm{M_t-\bar{M}_t}\cdot \norm{U_t}\\
            \leq & \sqrt{\frac{2}{\pi}}\delta\norm{U_t}\\
            \leq & \sqrt{\frac{2}{\pi}}\delta(\norm{E_t}+\norm{r_t}).
        \end{aligned}
    \end{equation}
    Here in (a) we used the contraction of projection. On the other hand,
    observing that $(I-u^{\star}u^{\star\top})(U_tU_t^{\top}-X^{\star})U_t=E_tU_t^{\top}U_t$, we have
    \begin{equation}
        \begin{aligned}
            \norm{E_{t+1}}&=\norm{E_t-\eta_t(I-u^{\star}u^{\star\top})M_tU_t}\\
            &\leq \norm{E_t-\eta_t(I-u^{\star}u^{\star\top})\bar{M}_tU_t}+\eta_t\norm{(I-u^{\star}u^{\star\top})(M_t-\bar{M}_t)U_t}\\
            &\leq \norm{E_t\left(I-\frac{\eta_t U_t^{\top}U_t}{\sqrt{\frac{\pi}{2}}\norm{\Delta_t}_F}\right)}+\sqrt{\frac{2}{\pi}}\delta\eta_t (\norm{r_t}+\norm{E_t})\\
            &\leq \norm{I-\frac{\eta_t U_t^{\top}U_t}{\sqrt{\frac{\pi}{2}}\norm{\Delta_t}_F}}\cdot \norm{E_t}+\sqrt{\frac{2}{\pi}}\delta\eta_t (\norm{r_t}+\norm{E_t}),
        \end{aligned}
    \end{equation}
    which completes the proof.
\end{proof}
Before presenting the proof of Proposition~\ref{error-dynamic}, we need the following intermediate result
\begin{lemma}\label{l_RIP}
    Suppose that the measurements satisfy sign-RIP with parameters $(\delta,r)$ and a constant scaling function $\varphi(X) = \sqrt{\frac{2}{\pi}}$. Then, for every $X\in\mathbb{S}_r$, we have 
    \begin{align}
        \left|\frac{1}{m}\norm{\mathcal{A}(X)}_1-\sqrt{\frac{2}{\pi}}\norm{X}_F\right|\leq \sqrt{\frac{2}{\pi}}\delta
    \end{align}
\end{lemma}
\begin{proof}
Due to the sign-RIP condition, we have $\norm{D-\sqrt{\frac{2}{\pi}}X}_F\leq \sqrt{\frac{2}{\pi}}\delta$
for every $X\in\mathbb{S}_r$ and $D\in \frac{1}{m}\sum_{i=1}^{m}\sign(\inner{A_i}{X})A_i$. This implies that
\begin{align}
    \sqrt{\frac{2}{\pi}}\delta= \sup_{Y\in\mathbb{S}}\inner{D-\sqrt{\frac{2}{\pi}}X}{Y}\geq \inner{D-\sqrt{\frac{2}{\pi}}X}{X} \geq \frac{1}{m}\norm{\mathcal{A}(X)}_1-\sqrt{\frac{2}{\pi}}\norm{X}_F.
\end{align}
Similarly, it can be shown that $\sqrt{\frac{2}{\pi}}\delta\geq -\frac{1}{m}\norm{\mathcal{A}(X)}_1+\sqrt{\frac{2}{\pi}}\norm{X}_F$. This completes the proof.
\end{proof}
Based on the above lemma, the sign-RIP condition implies the $\ell_1/\ell_2$-RIP condition. Given Lemmas~\ref{general-error-dynamics} and~\ref{l_RIP}, we are ready to present the proof of Proposition~\ref{error-dynamic}.
\begin{proof}[Proof of Proposition~\ref{error-dynamic}]
It is enough to show that the choice of $\eta_t=\frac{\pi}{2}\eta_0 \cdot \frac{1}{m}\sum\left|\inner{A_i}{U_t U_t^{\top}-X^{\star}}\right|$ in Lemma~\ref{general-error-dynamics} leads to the desired bounds. Based on Lemma~\ref{l_RIP} and $\delta\leq \frac{1}{2}$, we have
\begin{equation}\label{eq_eta}
    \eta_t\leq \sqrt{\frac{\pi}{2}}\eta_0\norm{\Delta_t}_F(1+\delta)\leq \sqrt{\frac{9\pi}{8}}\eta_0\norm{\Delta_t}_F.
\end{equation}
Similarly, we have the lower bound $\eta_t\geq \sqrt{\frac{\pi}{8}}\eta_0\norm{\Delta_t}_F$. Substituting these inequalities in Lemma~\ref{general-error-dynamics} leads to
\begin{equation}
    \begin{aligned}
        \norm{E_{t+1}}_F^2&\leq \norm{E_t}_F^2-\eta_0\norm{E_tU_t^{\top}}_F^2+3\delta\eta_0\norm{\Delta_t}_F\norm{E_tU_t^{\top}}_F+\frac{45}{2}\eta_0^2\left(\delta^2\norm{\Delta_t}_F^2+\norm{E_tU_t^{\top}}_F^2\right)\\
        &\stackrel{(a)}{\leq} \norm{E_t}_F^2+3\delta\eta_0\norm{\Delta_t}_F\norm{E_tU_t^{\top}}_F+\frac{45}{2}\delta^2\eta_0^2\norm{\Delta_t}_F^2\\
        &\stackrel{(b)}{\leq} \norm{E_t}_F^2+3\sqrt{10}\delta\eta_0\norm{E_t}_F\norm{U_t}_F+225\delta^2\eta_0^2\\
        &\stackrel{(c)}{\leq}\norm{E_t}_F^2+22\delta\eta_0\norm{E_t}_F+225\delta^2\eta_0^2\\
        &=\left(\norm{E_t}_F+11\delta\eta_0\right)^2+104\delta^2\eta_0^2.
    \end{aligned}
\end{equation}
where (a) follows from the assumption $\eta_0\leq \frac{2}{45}$, (b) follows form Fact \ref{claim-decomposition}, and (c) follows from Fact \ref{claim-U}.
Therefore, we have
\begin{equation}
    \norm{E_{t+1}}_F\leq \norm{E_t}_F+11\delta\eta_0+11\delta\eta_0=\norm{E_t}_F+22\delta\eta_0.
\end{equation}

Similarly, 
\begin{equation}
    \begin{aligned}
        \norm{E_{t+1}}&\leq \norm{I-\frac{\eta_t U_t^{\top}U_t}{\sqrt{\frac{\pi}{2}}\norm{\Delta_t}_F}}\cdot \norm{E_t}+\sqrt{\frac{2}{\pi}}\delta\eta_t (\norm{r_t}+\norm{E_t}).
    \end{aligned}
    \label{eq-95}
\end{equation}
Note that
\begin{equation}
    \norm{U_t^{\top}U_t}\leq \norm{U_t}^2\leq (\norm{E_t}+\norm{u^{\star}r_t^{\top}})^2\leq (\norm{E_t}+\norm{r_t})^2\leq 9.
\end{equation}
which, together with $\eta_t\leq \sqrt{\frac{9\pi}{8}}\eta_0$, implies 
\begin{align}
    \norm{\frac{\eta_t U_t^{\top}U_t}{\sqrt{\frac{\pi}{2}}\norm{\Delta_t}_F}}<1.
\end{align}
Therefore, the first term in the right hand side of the above inequality is upper bounded by $\|E_t\|$. On the other hand, the second term in~\eqref{eq-95} can be bounded as
\begin{align}\label{eq_upper}
    \sqrt{\frac{2}{\pi}}\delta\eta_t (\norm{r_t}+\norm{E_t})\leq \frac{3}{2}\delta\eta_0\|\Delta_t\|_F(\norm{r_t}+\norm{E_t})\leq 15\delta\eta_0
\end{align}
This completes the proof.
\end{proof}

\end{proof}

\subsection{Signal Dynamics}
\begin{proposition}[Signal Dynamics]
    \label{signal-dynamics}
   Under the conditions of Proposition~\ref{error-dynamic}, we have
    \begin{equation}
        \norm{r_{t+1}-\left(1+\eta_0(1-\norm{r_t}^2)\right)r_t}\leq 10\eta_0 \delta(\norm{E_t}+\norm{r_t})+2\eta_0\norm{E_t}^2\norm{r_t}.
        \label{eq:signal-dynamics}
    \end{equation}
\end{proposition}
\begin{proof}
Similarly, we first prove the following lemma which holds for a general choice of $\eta_t$.
\begin{lemma}
    \label{lem:general-signal-dynamic}
    Assuming $\norm{E_t}_F\leq 1, \norm{r_t}\leq 2$, we have
    \begin{equation}
        \begin{aligned}
            \norm{r_{t+1}-\left(1+\frac{\eta_t (1-\norm{r_t}^2)}{\sqrt{\frac{\pi}{2}}\norm{U_tU_t^{\top}-X^{\star}}_F}\right)r_t}&\leq \sqrt{\frac{2}{\pi}}\eta_t \delta(\norm{E_t}+\norm{r_t})\\&+\frac{\eta_t }{\sqrt{\frac{\pi}{2}}\norm{U_tU_t^{\top}-X^{\star}}_F}\norm{E_t}^2\norm{r_t}.
        \end{aligned}
    \end{equation}
\end{lemma}
\begin{proof}
    Recalling the notations $M_t\in \frac{1}{m}\sum_{i=1}^{m}\sign(\inner{A_i}{\Delta_t})A_i$ and $\bar{M}_t=\sqrt{\frac{2}{\pi}}\frac{\Delta_t}{\norm{\Delta_t}_F}$, we have
    \begin{equation}
        r_{t+1}=r_t -\eta_t U_t^{\top}M_t^{\top}u^{\star}.
    \end{equation}
    Therefore,
    \begin{equation}
        \begin{aligned}
            \norm{r_{t+1}-r_t+\eta_t\frac{U_t^{\top}\Delta_tu^{\star}}{\sqrt{\frac{\pi}{2}}\norm{\Delta_t}_F}}
            &\leq \eta_t\norm{U_t^{\top}(M_t-\bar{M}_t)^{\top}u^{\star}}\\
            &\leq \eta_t \norm{U_t}\cdot \norm{M_t-\bar{M}_t}\cdot \norm{u^{\star}}\\
            &\leq \sqrt{\frac{2}{\pi}}\eta_t \delta(\norm{E_t}+\norm{r_t}),
        \end{aligned}
    \end{equation}
    \begin{sloppypar}
    where the last inequality follows from the sign-RIP condition.
    On the other hand, since $U_{t}^{\top}\left(U_{t} U_{t}^{\top}-X^{\star}\right) u^{\star}=U_{t}^{\top} U_{t} r_{t}-r_{t}=\left(r_{t} r_{t}^{\top}+E_{t}^{\top} E_{t}\right) r_{t}-r_{t}=\left(\left\|r_{t}\right\|^{2}-1\right) r_{t}-E_{t}^{\top} E_{t} r_{t}$, one can write
    \end{sloppypar}
    \begin{equation}
        \begin{aligned}
            &\norm{r_{t+1}-\left(1+\frac{\eta_t (1-\norm{r_t}^2)}{\sqrt{\frac{\pi}{2}}\norm{\Delta_t}_F}\right)r_t}\\
            \leq & \sqrt{\frac{2}{\pi}}\eta_t \delta(\norm{E_t}+\norm{r_t})+\frac{\eta_t }{\sqrt{\frac{\pi}{2}}\norm{\Delta_t}_F}\norm{E_t^{\top}E_t r_t}\\
            \leq& \sqrt{\frac{2}{\pi}}\eta_t \delta(\norm{E_t}+\norm{r_t})+\frac{\eta_t }{\sqrt{\frac{\pi}{2}}\norm{\Delta_t}_F}\norm{E_t}^2\norm{r_t}.
        \end{aligned}
    \end{equation}
\end{proof}
Equipped with this lemma and~\eqref{eq_eta}, we write
\begin{equation}
    \begin{aligned}
        \norm{r_{t+1}-\left(1+\eta_0(1-\norm{r_t}^2)\right)r_t}&\leq \eta_0\delta(1-\norm{r_t}^2)\norm{r_t}\\
        &+ 2\eta_0\norm{\Delta_t}_F\delta(\norm{E_t}+\norm{r_t})\\
        &+2\eta_0\norm{E_t}^2\norm{r_t}.
    \end{aligned}
\end{equation}
Note that the first  term is dominated by the second term since $\norm{\Delta_t}_F\geq 1-\norm{r_t}^2$ due to~\eqref{eq_error_U} and the fact that $\norm{\Delta_t}_F\leq \sqrt{10}$. We finally have
\begin{equation}
    \begin{aligned}
        \norm{r_{t+1}-\left(1+\eta_0(1-\norm{r_t}^2)\right)r_t}\leq 10\eta_0 \delta(\norm{E_t}+\norm{r_t})+2\eta_0\norm{E_t}^2\norm{r_t}.
    \end{aligned}
\end{equation}

\end{proof}

\subsection{Convergence Result}
Now we can formally state the following convergence theorem for the noiseless setting.
\begin{theorem}
    \label{theorem:noiseless}
    Assume that the measurements are noiseless and satisfy the sign-RIP condition with parameters $\left(\min\{r'+1,d\},\delta\right)$, $\delta\lesssim 1$, and constant scaling function $\varphi(X) = \sqrt{\frac{2}{\pi}}$. Suppose that $\alpha \asymp \sqrt{\delta}/\sqrt[4]{r'}$ and the step size $\eta_t$ is chosen as~\eqref{eq_stepsize} with $\eta_0\lesssim 1$. Then, after $T\asymp\frac{\log\left({r'}/{\delta}\right)}{\eta_0}$ iterations, we have
    \begin{equation}
        \norm{U_TU_T^{\top}-X^{\star}}^2_F\lesssim \delta^2\log^2\left(\frac{r'}{\delta}\right).
    \end{equation}
\end{theorem}
\begin{proof}
To start the proof, we first provide the following intermediate lemma on the initial signal and error terms.

\begin{lemma}
    Suppose that $U_0=\alpha B$ is chosen by Algorithm~\ref{alg::spectral-initialization}. Then under the conditions of Theorem~\ref{theorem:noiseless} , we have
\begin{equation}
    \norm{r_0}=\alpha (1\pm O(\sqrt{\delta})), \quad \norm{E_0}=  O(\alpha\sqrt{\delta}), \quad \norm{E_0}_F=O(\alpha\sqrt[4]{r'}\sqrt{\delta}).
\end{equation}
\label{lem::initialization_scale_noseless}
\end{lemma}
\begin{proof}
    For convenience, we list the diagonal elements of $\Sigma$ in a descending order $\sigma_1\geq \sigma_2\geq\cdots \geq \sigma_{d}$. Moreover, suppose that $u_1,\cdots, u_d$ are the corresponding eigenvectors. For simplicity, we define $\sigma'_i=\max\{\sigma_i,0\}$. Due to sign-RIP, we have
    \begin{equation}
    \norm{C-\sqrt{\frac{2}{\pi}}X^{\star}}_F\lesssim\delta.
    \end{equation}
    Therefore, we have
    \begin{equation}
        \begin{aligned}
             \norm{\hat{X}-X^{\star}}_F&=\norm{\frac{C}{\norm{C}_F}-X^{\star}}_F\\&\leq \norm{\sqrt{\frac{\pi}{2}}C-X^{\star}}_F+\norm{X^{\star}}_F-\sqrt{\frac{\pi}{2}}\norm{C}_F\\&=O(\delta),
        \end{aligned}
    \end{equation}
    here we use triangle inequality twice.
    Since $\mathrm{span}\{u_1,\cdots,u_d\}=\R^d$ and $\norm{u^{\star}}=1$, we can write $u^{\star}=\sum_{i=1}^{d}\beta_i u_i$, where $\sum_{i=1}^{d}\beta_i^2=1$.
    Therefore, we have
    \begin{equation}
        \begin{aligned}
            \norm{\hat{X}-X^{\star}}_F^2&=\inner{\sum_{i=1}^{d}\sigma_iu_iu_i^{\top}-u^{\star}u^{\star\top}}{\sum_{i=1}^{d}\sigma_iu_iu_i^{\top}-u^{\star}u^{\star\top}}\\
            &=\sum_{i=1}^{d}\sigma^2-2\sum_{i=1}^{d}\sigma_i \beta_i^2+1\\
            &=\sum_{i=1}^{d}\left(\sigma_i-\beta_i^2\right)^2+1-\sum_{i=1}^{d}\beta_i^4\lesssim \delta^2,
        \end{aligned}
    \end{equation}
    which implies
    \begin{equation}
        \sum_{i=1}^{d}\left(\sigma_i-\beta_i^2\right)^2=O\left(\delta^2\right), \quad \sum_{i=1}^{d}\beta_i^4 = 1-O\left(\delta^2\right).
    \end{equation}
    Without loss of generality, we define $\beta_{\max}=|\beta_1|=\max\{|\beta_i|:1\leq i\leq d\}$. Therefore, we have
    \begin{equation}
        1-O\left(\delta^2\right)=\sum_{i=1}^{d}\beta_i^4\leq \beta_{\max}^2\sum_{i=1}^{d}\beta_i^2=\beta_{\max}^2.
    \end{equation}
    Here we use the fact that $\sum_{i=1}^{d}\beta_i^2=1$. Moreover, it is easy to see that $\sum_{i=2}^{d}\beta_i^2=1-\beta_{\max}^2=O\left(\delta^2\right)$.
    Therefore, we have
    \begin{equation}
        \begin{aligned}
            \norm{BB^{\top}-X^{\star}}_F^2&=\sum_{i=1}^{r'}\sigma_i^{'2}-\sum_{i=1}^{r'}\sigma'_i\beta_i^2+1\\
            &=\sum_{i=1}^{r'}\left(\sigma'_i-\beta_i^2\right)^2-\sum_{i=1}^{r'}\beta_i^4+1\\
            &=1-\beta_{\max}^4+O\left(\delta^2\right)\\
            &=\left(1+\beta_{\max}^2\right)\left(1-\beta_{\max}^2\right)+O\left(\delta^2\right)=O\left(\delta^2\right).
        \end{aligned}
    \end{equation}
    By Lemma~\ref{decomposition-error-signal}, we immediately have 
    \begin{equation}
        \left(1-\norm{B^{\top}u^{\star}}^2\right)^2\leq \norm{BB^{\top}-X^{\star}}_F^2\lesssim \delta^2.
    \end{equation}
    On the other hand, note that $\norm{r_0}=\alpha \norm{B^{\top}u^{\star}}$, which together with the above inequality, implies $\norm{r_0}=\alpha (1\pm O(\sqrt{\delta}))$.
    Similarly, we have
    \begin{equation}
        \begin{aligned}
            \norm{(I-u^{\star}u^{\star\top})B}^2&=\sup_{\norm{x}\leq 1} x^{\top} (I-u^{\star}u^{\star\top})BB^{\top}(I-u^{\star}u^{\star\top})x\\
            &=\sup_{\norm{x}\leq 1} x^{\top} (I-u^{\star}u^{\star\top})(BB^{\top}-X^{\star})(I-u^{\star}u^{\star\top})x\\
            &\leq \norm{BB^{\top}-X^{\star}}\leq \norm{BB^{\top}-X^{\star}}_F= O(\delta),
        \end{aligned}
    \end{equation}
    which leads to $\norm{E_0}=O(\alpha\sqrt{\delta})$. As for the Frobenius norm, we have
    \begin{equation}
        \begin{aligned}
            \norm{(I-u^{\star}u^{\star\top})B}^2_F&=\inner{(I-u^{\star}u^{\star\top})B}{(I-u^{\star}u^{\star\top})B}\\
            &=\inner{I-u^{\star}u^{\star\top}}{BB^{\top}-X^{\star}}\\
            &=\sum_{i=1}^{r'}\sigma'_i-\sum_{i=1}^{r'}\sigma'_i\beta_i^2\\
            &\leq \sigma'_1\left(1-\beta_{\max}^2\right)+\sum_{i=2}^{r'}\sigma'_i\\
            &=\sum_{i=2}^{r'}\left(\sigma'_i-\beta_i^2\right)+\sum_{i=2}^{r'}\beta_i^2+O\left(\delta^2\right)\\
            &\leq \sqrt{r'-1}\sqrt[]{\sum_{i=2}^{r'}\left(\sigma'_i-\beta_i^2\right)^2}+1-\beta_{\max}^2+O\left(\delta^2\right)\\
            &=O\left(\sqrt{r'}\delta\right).
        \end{aligned}
    \end{equation}
    Therefore, we have $\norm{E_0}_F=O\left(\alpha\sqrt[4]{r'}\sqrt{\delta}\right)$, which completes the full proof.
\end{proof}

Based on the above lemma, the assumptions of Propositions \ref{error-dynamic} and \ref{signal-dynamics} are valid for the base case.
Next, we control the Frobenius norm of the error term. 
    From Proposition~\ref{error-dynamic}, we have
    \begin{equation}
        \norm{E_{t+1}}_F\leq \norm{E_t}_F+22\delta\eta_0,
    \end{equation}
    which implies
    \begin{equation}
        \begin{aligned}
            \norm{E_T}_F=\norm{E_0}_F+\sum_{t=1}^{T}\left(\norm{E_t}_F-\norm{E_{t-1}}_F\right)\leq \norm{E_0}_F+22\delta\eta_0 T.
        \end{aligned}
    \end{equation}

    Therefore, since $T\asymp\log\left(\frac{1}{\delta}\right)/\eta_0$ and $\alpha\asymp\sqrt{\delta}/\sqrt[4]{r'}$, we have $\norm{E_T}_F\lesssim \delta\log\frac{r'}{\delta}$.
    Therefore, we have
    $\norm{E_T}\leq \norm{E_T}_F\lesssim\delta\log\frac{r'}{\delta}$.
    This shows that the error term remains small throughout the iterations of SubGD.
    Without loss of generality and to simplify our subsequent analysis, we assume that $\norm{E_T}\leq \delta\log\frac{r'}{\delta}$, which can be ensured with sufficiently small $\eta_0$.
    Next, we control the signal term. Due to Proposition~\ref{signal-dynamics}, we have
    \begin{equation}
        \begin{aligned}
            \norm{r_{t+1}}&\geq (1+\eta_0(1-\norm{r_t}^2))\norm{r_t}-10\eta_0 \delta(\norm{E_t}+\norm{r_t})-2\eta_0\norm{E_t}^2\norm{r_t}.
        \end{aligned}
        \label{signal-102}
    \end{equation}
    
    Now, we separate our analysis into two stages. In the first stage, we show that the signal grows at a linear rate, provided that $\|r_t\|\leq 1/2$. To show this, we first prove that during the whole training process, the signal term is always larger than the error term.

\begin{lemma}
    Suppose that $\delta\leq 1/50$. Then, for any $0\leq t\leq T=\Theta\left(\log\frac{1}{\alpha}/\eta_0\right)$, we have
    \begin{equation}
        \norm{E_t}\leq \norm{r_t}.
        \label{E_t-r_t}
    \end{equation}
    \label{lemma::E_t-r_t}
\end{lemma}
\begin{proof}
    We prove this lemma by induction. For the base case, \eqref{E_t-r_t} holds since we have $\norm{r_0}=\alpha (1\pm O(\sqrt{\delta})), \norm{E_0}= \alpha O(\sqrt{\delta})$. Now, suppose that \eqref{E_t-r_t} holds at time $t$. Based on \eqref{eq-95} and~\eqref{eq_upper}, we have
    \begin{equation}
        \norm{E_{t+1}}\leq (1+5\eta_0\delta)\norm{E_t}+5\eta_0\delta\norm{r_t}\leq (1+10\eta_0\delta)\norm{r_t}.
    \end{equation}
    On the other hand, due to \eqref{signal-102}, we have
    \begin{equation}
        \begin{aligned}
            \norm{r_{t+1}}&\geq(1+\eta_0(1-\norm{r_t}^2))\norm{r_t}-10\eta_0 \delta(\norm{E_t}+\norm{r_t})-2\eta_0\norm{E_t}^2\norm{r_t}\\
            &\geq (1+\eta_0(1-3\norm{r_t}^2))\norm{r_t}-20\delta\eta_0\norm{r_t}\\
            & \geq\left(1+\frac{1}{5}\eta_0\right)\norm{r_t}.
        \end{aligned}
    \end{equation}
    Here we used the induction hypothesis $\norm{E_t}\leq \norm{r_t}$.
    The above two inequalities, together with $\delta\leq 1/50$, imply that $\|E_{t+1}\|\leq \|r_{t+1}\|$.
\end{proof}

    During the proof of the above lemma, we showed that
    \begin{equation}
        \norm{r_{t+1}}\geq (1+\eta_0/5)\norm{r_t}.
    \end{equation}
    provided that $\delta\leq 1/50$. Now, assuming that $T_1\gtrsim \log \frac{1}{\alpha}/\eta_0$, we have
    \begin{equation}
        \begin{aligned}
            \norm{r_{T_1}}&\geq (1-O(\sqrt{\delta}))\alpha (1+\eta_0/5)^{T_1}\geq \frac{1}{2},
        \end{aligned}
    \end{equation}
    This implies that, after $T_1$ iterations, the signal term will have a norm of at least $1/2$.
    In the second stage, we assume that $1\geq \|r_t\|\geq 1/2$. One can write
\begin{equation}\label{eq_x}
    \begin{aligned}
        \norm{r_{t+1}}&\geq (1+\eta_0(1-\norm{r_t}^2))\norm{r_t}-10\eta_0 \delta(\norm{E_t}+\norm{r_t})-2\eta_0\norm{E_t}^2\norm{r_t}\\
        &\geq(1+\eta_0(1-\norm{r_t}))\norm{r_t}-20\eta_0 \delta\norm{r_t}-4\eta_0\delta^2\log^2\frac{r'}{\delta},
    \end{aligned}
\end{equation}
where we used $1-\norm{r_t}^2\geq 1-\norm{r_t}$ given $\norm{r_t}\leq 1$, and Lemma~\ref{lemma::E_t-r_t}.
    
    For the sake of simplicity, we define $x_t=1-\norm{r_{t+T_1}}$. Hence, \eqref{eq_x} can be simplified as
    \begin{equation}
        \begin{aligned}
            x_{t+1}&\leq 1-(1-20\eta_0\delta+\eta_0x_t)(1-x_t)+4\eta_0\delta^2\log^2 \frac{r'}{\delta}\\
            & \leq(1-\eta_0+20\eta_0\delta)x_t+\eta_0x_t^2+20\eta_0\delta+4\eta_0\delta^2\log^2 \frac{r'}{\delta}\\
            &\leq (1-\frac{3}{4}\eta_0)x_t+\frac{1}{2}\eta_0x_t+20\eta_0\delta\left(1+\delta\log^2\frac{r'}{\delta}\right)\\
            &\leq (1-\eta_0/4)x_t+20\eta_0\delta\left(1+\delta\log^2\frac{r'}{\delta}\right),
        \end{aligned}
    \end{equation}
    Here, we used $x_t\leq 1/2$ and $\delta\leq 1/80$.
    Then, we have
    \begin{equation}
        x_{t+1}-80\delta\left(1+\delta\log^2\frac{1}{\delta}\right)\leq \left(1-\frac{\eta_0}{4}\right)\left(x_t-80\delta\left(1+\delta\log^2\frac{r'}{\delta}\right)\right),
    \end{equation}
    which implies
    \begin{equation}
        x_{T_2} \leq 80\delta\left(1+\delta\log^2\frac{r'}{\delta}\right) + \frac{1}{2}\left(1-\eta_0/4\right)^{T_2}.
    \end{equation}
    Upon choosing $T_2\gtrsim \log \frac{1}{\delta}/\eta_0$, we have $x_{T_2}\lesssim \delta\vee \delta^2\log^2\frac{r'}{\delta}$, which is equivalent to $\norm{r_{T_1+T_2}}\geq 1-O\left(\delta+ \delta^2\log^2\frac{r'}{\delta}\right)$.
    
    This completes the proof under the assumption $\norm{r_{T_1+T_2}}\leq 1$. Now, it remains to show that the error bound holds even if $\norm{r_{T_1+T_2}}> 1$. To this goal, first we show that $T_3=\Omega\left(\log\frac{r'}{\delta}/\eta_0\right)$ is necessary to guarantee the convergence of SubGD. In particular, we prove that we need at least $T_3=\Omega\left(\log\frac{r'}{\delta}/\eta_0\right)$ to ensure $\norm{r_{T_3}}\geq \frac{1}{2}$. To this goal, suppose that $\|r_t\|\leq 1/2$ for every $t\leq T$. Due to Proposition~\ref{signal-dynamics}, we have
\begin{equation}
    \begin{aligned}
        \norm{r_{t+1}}&\leq (1+\eta_0(1-\norm{r_t}^2))\norm{r_t}+10\eta_0 \delta(\norm{E_t}+\norm{r_t})+2\eta_0\norm{E_t}^2\norm{r_t}.\\
        &\stackrel{(a)}{\leq} (1+20\eta_0\delta+\eta_0(1-\norm{r_t}^2))\norm{r_t}+2\eta_0\norm{r_t}^2\cdot\norm{r_t}\\
        &\stackrel{(b)}{\leq} \left(1+20\eta_0\delta+\eta_0+\frac{1}{2}\eta_0\right)\norm{r_t}\\
        &\leq (1+2\eta_0)\norm{r_t}.
    \end{aligned}
\end{equation}

Here we used Lemma \ref{lemma::E_t-r_t} and $\norm{r_t}\leq \frac{1}{2}$ in (a) and (b), respectively. Therefore, 
\begin{equation}
    \norm{r_T}\leq \alpha (1+2\eta_0)^T.
\end{equation}
This shows that we need at least $T_3=\Omega\left(\log\frac{1}{\alpha}/\eta_0\right)=\Omega\left(\log\frac{r'}{\delta}/\eta_0\right)$ iterations to guarantee $\|r_t\|\geq 1/2$. 
    Now, suppose $\norm{r_{T_1+T_2}}> 1$. Without loss of generality, we assume that $\norm{r_{T_1+T_2-1}}\leq 1<\norm{r_{T_1+T_2}}$ (since $T_3$ and $T_1+T_2$ have the same order). Under this assumption, we show that $\norm{r_{T_1+T_2}}\leq 1+O\left(\delta+ \delta^2\log^2\frac{r'}{\delta}\right)$. By Proposition~\ref{error-dynamic}, we have
    \begin{equation}
        \begin{aligned}
            \norm{r_{t+1}}-\norm{r_t}&\leq \eta_0(1-\norm{r_t}^2)\norm{r_t}+10\eta_0 \delta(\norm{E_t}+\norm{r_t})+2\eta_0\norm{E_t}^2\norm{r_t}\\
            &\leq 6\eta_0(1-\norm{r_t})+40\eta_0\delta+4\eta_0\delta^2\log^2\frac{r'}{\delta}.
        \end{aligned}
    \end{equation}
    where we used the Lemma~\ref{lemma::E_t-r_t} and $\norm{r_t}\leq 2$. Then, by our choice of $\norm{r_{T_1+T_2-1}}$ and $\norm{r_{T_1+T_2}}$, we have
    \begin{equation}
        \begin{aligned}
            \norm{r_{T_1+T_2}}-\norm{r_{T_1+T_2-1}}&\leq 6\eta_0(1-\norm{r_{T_1+T_2-1}})+40\eta_0\delta+4\eta_0\delta^2\log^2\frac{r'}{\delta}\\
            &\leq 6\eta_0\left(\norm{r_{T_1+T_2}}-\norm{r_{T_1+T_2-1}}\right)+40\eta_0\delta+4\eta_0\delta^2\log^2\frac{r'}{\delta}.
        \end{aligned}
    \end{equation}
    Then, since $\eta_0\lesssim 1$, we have
    \begin{equation}
        \norm{r_{T_1+T_2}}-1\leq \norm{r_{T_1+T_2}}-\norm{r_{T_1+T_2-1}}\lesssim \delta\vee\delta^2\log^2\frac{r'}{\delta}.
    \end{equation}
    This implies that $|1-\norm{r_{T_1+T_2}}|\lesssim \delta\vee\delta^2\log^2\frac{r'}{\delta}$.

    Finally, these two stages characterize the behavior of $r_t$ and its convergence to the true solution. In particular, with the choice of  $T=T_1+T_2=O(\log\frac{r'}{\delta}/\eta_0)$,  and according to Lemma \ref{decomposition-error-signal}, we have
    \begin{equation}
        \begin{aligned}
            \norm{U_TU_T^{\top}-X^{\star}}^2_F&\leq (1-\norm{r_T}^2)^2+2\norm{E_T}^2\norm{r_T}^2+\norm{E_T}_F^4\\
            &\lesssim \delta^2+\delta^4 \log^4 \frac{r'}{\delta} + \delta^2\log^2\frac{r'}{\delta}+ \delta^4 \log^4 \frac{r'}{\delta}\\
            &\lesssim \delta^2\log^2\frac{r'}{\delta},
        \end{aligned}
    \end{equation}
which completes the proof.  $\hfill\square$

\end{proof}
\section*{Proofs for the Noisy Case}
For simplicity of notation, we denote $\varphi_t = \varphi(\Delta_t)$, where $\Delta_t = U_tU_t^\top - X^*$.
\subsection{Proof of Proposition \ref{error-dynamics-noisy}}
Analogous to the proof of Proposition~\ref{error-dynamic}, first we provide a general result which holds for arbitrary learning rates. 
\begin{lemma}
    \label{error-dynamics-general}
    Suppose that $\norm{E_t}_F\leq 1,\norm{r_t}\leq 2$, $\delta\leq \frac{1}{2}$, then, the following inequalities hold
    \begin{equation}
        \norm{E_{t+1}}_F^2\leq \norm{E_t}_F^2+2\eta_t\varphi_t\left(-\frac{\norm{E_tU_t^{\top}}_F^2}{\norm{\Delta_t}_F}+\delta \norm{E_tU_t^{\top}}_F\right)+10\eta_t^2\varphi_t^2\left(\delta^2+\frac{\norm{E_tU_t^{\top}}_F^2}{\norm{\Delta_t}_F^2}\right),
    \end{equation}
    \begin{equation}
        \norm{E_{t+1}}\leq \norm{I-\frac{\eta_t\varphi_t U_t^{\top}U_t}{\norm{\Delta_t}_F}}\cdot \norm{E_t}+\delta\eta_t\varphi_t (\norm{r_t}+\norm{E_t}).
    \end{equation}
\end{lemma}

\begin{proof}
The proof is similar to that of Lemma~\ref{general-error-dynamics}. The details are omitted for brevity.
\end{proof}

Now, we are ready to present the proof of Proposition~\ref{error-dynamics-noisy}.

\begin{proof}[Proof of Proposition~\ref{error-dynamics-noisy}]

Based on the sign-RIP condition, the step sizes satisfy
\begin{equation}
    \eta_t=\frac{\eta_0\rho^t}{\norm{D}_F}\leq \frac{\eta_0\rho^t}{\varphi_t(1-\delta)}\leq \frac{2\eta_0\rho^t}{\varphi_t},
\end{equation}
where $D\in\mathcal{M}(U_tU_t-X^*)$.
For the Frobenius norm, we have
\begin{equation}
    \begin{aligned}
        \norm{E_{t+1}}_F^2&\leq \norm{E_t}_F^2+2\eta_t\varphi_t\left(-\frac{\norm{E_tU_t^{\top}}_F^2}{\norm{\Delta_t}_F}+\delta \norm{E_tU_t^{\top}}_F\right)+10\eta_t^2\varphi_t^2\left(\delta^2+\frac{\norm{E_tU_t^{\top}}_F^2}{\norm{\Delta_t}_F^2}\right)\\
        &\leq \norm{E_t}_F^2+2\eta_t\varphi_t\delta\norm{E_tU_t^{\top}}_F+10\delta^2\eta_t^2\varphi_t^2\\
        &\leq \norm{E_t}_F^2+4\delta\eta_0\rho^t\norm{E_tU_t^{\top}}_F+20\delta^2\eta_0^2\rho^{2t}.
    \end{aligned}
\end{equation}
where in the second inequality, we used the assumption $\eta_0\lesssim \delta\lesssim\|\Delta_t\|_F$, which implies
\begin{align}
    -2\eta_t\varphi_t\frac{\norm{E_tU_t^{\top}}_F^2}{\norm{\Delta_t}_F}+10\eta_t^2\varphi_t^2\frac{\norm{E_tU_t^{\top}}_F^2}{\norm{\Delta_t}_F^2}\leq 0
\end{align}
Furthermore, note that
\begin{equation}
    \begin{aligned}
        \norm{E_tU_t^{\top}}_F^2&=\norm{E_tE_t^{\top}}_F^2+\norm{E_tr_tu^{\star\top}}_F^2\\&\leq \norm{E_t}_F^4+\norm{E_t}_F^2\norm{r_t}^2\\&\leq (1+4)\norm{E_t}_F^2,
    \end{aligned}
\end{equation}
which implies
\begin{equation}
    \begin{aligned}
        \norm{E_{t+1}}_F^2&\leq \norm{E_t}_F^2+20\delta\eta_0\rho^t\norm{E_t}_F+20\delta^2\eta_0^2\rho^{2t}\\&\leq \left(\norm{E_t}_F+10\delta\eta_0\rho^t\right)^2.
    \end{aligned}
\end{equation}
This leads to $\norm{E_{t+1}}_F\leq \norm{E_t}_F+10\delta\eta_0\rho^t$.

For the spectral norm, since we suppose $\eta_0\lesssim \delta\lesssim\|\Delta_t\|_F$, we have $\norm{I-\frac{\eta_t\varphi_t U_t^{\top}U_t}{\norm{\Delta_t}_F}}\leq 1$. Combined with Lemma~\ref{error-dynamics-general}, this implies that
\begin{equation}
    \begin{aligned}
        \norm{E_{t+1}}&\leq \norm{E_t}+\delta\eta_t\varphi_t (\norm{r_t}+\norm{E_t})\leq \norm{E_t}+2\delta\eta_0\rho^t (\norm{r_t}+\norm{E_t}).
    \end{aligned}
\end{equation}
thereby completing the proof.
\end{proof}

\subsection{Proof of Proposition \ref{signal-dynamics-noisy}}
Similar to the proof of Proposition~\ref{signal-dynamics}, we first present a general result which holds for arbitrary learning rates.

\begin{lemma}
    \label{signal-dynamics-noisy-general}
    For any learning rate $\eta_t$, if $\norm{E_t}_F\leq 1, \norm{r_t}\leq 2$, then we have
    \begin{equation}
        \begin{aligned}
            &\norm{r_{t+1}-\left(1+\frac{\varphi_t\eta_t (1-\norm{r_t}^2)}{\norm{U_tU_t^{\top}-X^{\star}}_F}\right)r_t}\\\leq& \delta\eta_t \varphi_t(\norm{E_t}+\norm{r_t})+\frac{\varphi_t\eta_t }{\norm{U_tU_t^{\top}-X^{\star}}_F}\norm{E_t}^2\norm{r_t}.
        \end{aligned}
    \end{equation}
\end{lemma}

\begin{proof}
    The proof is similar to that of Lemma~\ref{lem:general-signal-dynamic}. The details are omitted for brevity.
\end{proof}

\begin{proof}[Proof of Proposition~\ref{signal-dynamics-noisy}]
Assuming that $\delta\leq \frac{1}{2}$, we have 
\begin{equation}
    \left|\eta_t-\frac{\eta_0}{\varphi_t}\rho^t\right| = \left|\frac{\eta_0}{\norm{D}_F}\rho^t-\frac{\eta_0}{\varphi_t}\rho^t\right|\leq \frac{\delta \varphi_t\eta_0\rho^t}{(1-\delta)\varphi_t^2}\leq \frac{2\delta\eta_0\rho^t}{\varphi_t}.
\end{equation}
Combined with Lemma~\ref{signal-dynamics-noisy-general}, this implies that
\begin{equation}
    \begin{aligned}
    \norm{r_{t+1}-\left(1+\frac{\eta_0\rho^t (1-\norm{r_t}^2)}{\norm{U_tU_t^{\top}-X^{\star}}_F}\right)r_t}&\leq 2\delta\eta_0\rho^t(\norm{E_t}+\norm{r_t})+\frac{2\eta_0\rho^t }{\norm{U_tU_t^{\top}-X^{\star}}_F}\norm{E_t}^2\norm{r_t}\\
    &+\frac{2\delta\eta_0\rho^t}{\norm{U_tU_t^{\top}-X^{\star}}_F}(1-\norm{r_t}^2)\norm{r_t}.
    \end{aligned}
\end{equation}
which completes the proof.
\end{proof}

\subsection{Proof of Theorem \ref{convergence-theorem-noisy}}
Recall that $T=\Theta(\log\frac{1}{\alpha}/\eta_0)$.
First, we show that the error term remains small during the iterations of SubGD.
Proposition \ref{error-dynamics-noisy} leads to
    \begin{equation}\label{eq_error_noisy}
        \norm{E_t}\leq\norm{E_t}_F=\norm{E_0}_F+\sum_{t=1}^{t}\left(\norm{E_t}_F-\norm{E_{t-1}}_F\right)\leq \sqrt[4]{r'}\sqrt{\delta}\alpha+10\delta\eta_0T\lesssim\delta\log\frac{r'}{\delta},
    \end{equation}
    provided that $\norm{\Delta_t}_F\geq 1-\norm{r_t}^2\gtrsim \delta$. To verify this assumption, we show that $\norm{\Delta_t}_F\geq 1-\norm{r_t}^2\gtrsim \delta\log\frac{r^{\prime}}{\delta}$ for every $t\leq\bar{T}$, where $\bar{T}\gtrsim \log\frac{1}{\alpha}/\eta_0$. To this goal, first we present a preliminary claim
    \begin{claim}
    \label{claim::E_r}
        For every $0\leq t\leq T$, we have $\norm{E_t}\leq \norm{r_t}$.
    \end{claim}
    \begin{proof}
       It follows an argument analogous to the proof of Lemma~\ref{lemma::E_t-r_t}. The details are omitted for brevity.
    \end{proof}
    Based on this claim, we are ready to show that $\norm{\Delta_t}_F\geq 1-\norm{r_t}^2\gtrsim \delta\log\frac{r^{\prime}}{\delta}$ for every $t\leq\bar{T}$, where $\bar{T}\gtrsim \log\frac{1}{\alpha}/\eta_0$.
    \begin{claim}\label{claim1}
    Suppose that $\delta\leq 1/6$. Then, for every $0\leq t\lesssim \log\frac{1}{\alpha}/\eta_0$, we have $\norm{r_t}\leq \frac{1}{2}$. 
    \end{claim}
    \begin{proof}
       The statement holds for $t = 0$ since $\norm{r_0}=\Theta(\alpha)$. Now, suppose that $\norm{r_t}\leq \frac{1}{2}$. Then, we have 
       \begin{equation}
        \begin{aligned}
            \norm{r_{t+1}}&\leq \left(1+\frac{4}{3}\frac{\eta_0\rho^t}{\norm{\Delta_t}_F}(1-\norm{r_t}^2)\right)\norm{r_t}+2\delta\eta_0\rho^t(\norm{E_t}+\norm{r_t})+\frac{2\eta_0\rho^t }{\norm{\Delta_t}_F}\norm{E_t}^2\norm{r_t}\\
            &\stackrel{(a)}{\leq}  \left(1+O\left(1\right)\eta_0\rho^t\right)\norm{r_t}+2\delta\eta_0\rho^t(\norm{E_t}+\norm{r_t})\\&\stackrel{(b)}{\leq} \left(1+O(1)\eta_0\rho^t\right)\norm{r_t}.
        \end{aligned}
        \label{eq::141}
    \end{equation}
    where in (a) we use the fact that $\norm{\Delta_t}_F\geq 1-\norm{r_t}^2\geq 3/4$ and $\norm{E_t}\lesssim 1$; and in (b) we use Claim~\ref{claim::E_r}. Without loss of generality, we assume that $\norm{r_{t+1}}\leq \left(1+\eta_0\rho^t\right)\norm{r_t}$. Hence, it suffices to show that
    \begin{equation}
        \norm{r_0}\prod_{s=1}^{t}\left(1+\eta_0\rho^s\right)=\Theta(\alpha) \prod_{s=1}^{t}\left(1+\eta_0\rho^s\right)\leq \frac{1}{2},
    \end{equation}
    for every $0\leq t\lesssim \log\frac{1}{\alpha}/\eta_0$. This is equivalent to
    \begin{equation}
        \sum_{s=1}^{t}\log\left(1+\eta_0\rho^s\right)\leq \log\frac{1}{2\alpha}.
    \end{equation}
    On the other hand, note that 
    \begin{equation}
        \sum_{s=1}^{t}\log\left(1+\eta_0\rho^s\right)\leq \sum_{s=1}^{t}\eta_0\rho^s\leq\eta_0\frac{1-\rho^t}{1-\rho}\leq C\log\frac{1}{\alpha}\left(1-\rho^t\right).
    \end{equation}
    Therefore, to finish the proof, we need to show that $C\log\frac{1}{\alpha}\left(1-\rho^t\right)\leq \log\frac{1}{2\alpha}$, which implies $1-\frac{1}{C}+\frac{\log 2}{C\log\frac{1}{\alpha}}\leq \rho^T$. This can be easily verified for every $t\lesssim \log\frac{1}{\alpha}/\eta_0$, by noting that $\rho=1-\Theta(\eta_0/\log\frac{1}{\alpha})$.
    \end{proof}
    

Based on the above claim and upon choosing $\bar{T}\asymp \log\frac{1}{\alpha}/\eta_0$, the error term is bounded as~\eqref{eq_error_noisy} for every $t\leq \bar{T}$. Now, note that the proof is completed if $\norm{\Delta_t}_F\lesssim\delta\log\frac{r^{\prime}}{\delta}$ for some $\bar{T}\leq t\leq T$. Therefore, suppose that $\norm{\Delta_t}_F\gtrsim\delta\log\frac{r^{\prime}}{\delta}$ for every $\bar{T}\leq t\leq T$. This implies that the error bound~\eqref{eq_error_noisy} holds for every $\bar{T}\leq t\leq T$. Moreover, we assume that $1-\norm{r_t}^2\geq 3 \norm{E_t}_F$, since otherwise, we have $1-\norm{r_t}^2\lesssim \delta\log(r'/\delta)$, and the proof is completed together with $\|E_t\|_t\leq \delta\log(r'/\delta)$ and Lemma \ref{decomposition-error-signal}. This leads to
    \begin{equation}\label{eq_eup_low}
        1-\norm{r_t}^2\leq\norm{\Delta_t}_F\leq 1-\norm{r_t}^2+\norm{E_t}\norm{r_t}+\norm{E_t}_F^2\leq \frac{13}{9}(1-\norm{r_t}^2).
    \end{equation}
    assuming that $\norm{r_t}\leq 1$. Then, according to Proposition \ref{signal-dynamics-noisy}, we have
    \begin{equation}
        \begin{aligned}
            \norm{r_{t+1}}&\geq \left(1+\frac{2}{3}\frac{\eta_0\rho^t}{\norm{\Delta_t}_F}(1-\norm{r_t}^2)\right)\norm{r_t}-2\delta\eta_0\rho^t(\norm{E_t}+\norm{r_t})-\frac{2\eta_0\rho^t }{\norm{\Delta_t}_F}\norm{E_t}^2\norm{r_t}\\&\stackrel{(a)}{\geq} \left(1+\Omega(1)\eta_0\rho^t\right)\norm{r_t}-2\delta\eta_0\rho^t\norm{E_t}.
        \end{aligned}
    \end{equation}
    where in (a) we used $\norm{E_t}^2\leq (1-\norm{r_t}^2)/9$, inequality~\eqref{eq_eup_low}, and $\delta\lesssim 1$. To proceed, note that $\norm{E_t}\leq \norm{r_t}$ due to Claim~\ref{claim::E_r}. Hence,
    we have
    \begin{equation}
     \norm{r_{t+1}}\geq \left(1+\Omega(1)\eta_0\rho^t\right)\norm{r_t}.
    \end{equation}
    for every $0\leq t\leq T$. Now, it remains to show that after $T=O\left(\log\left(\frac{1}{\alpha}\right)/\eta_0\right)$ iterations, the signal term approaches $1$. 
    Without loss of generality, we assume that $\norm{r_{t+1}}\geq (1+\eta_0\rho^t)\norm{r_t}$, which implies $\norm{r_{T}}\geq\alpha \prod_{t=1}^{T}\left(1+\eta_0\rho^t\right)$. Taking the logarithm of the right hand side leads to
    \begin{equation}
        \begin{aligned}
            \sum_{t=1}^{T}\log\left(1+\eta_0\rho^t\right)\geq\sum_{t=1}^{T}\frac{\eta_0\rho^t}{1+\eta_0\rho^t}\geq \frac{\eta_0}{2}\frac{1-\rho^{T}}{1-\rho}.
        \end{aligned}
    \end{equation}
    where we used the lower bound $\log(1+x)\geq \frac{x}{1+x}$ for $x\geq -1$. Now, upon defining $\gamma=1-\rho$, we have
    \begin{equation}
        \begin{aligned}
            \frac{\eta_0}{2}\frac{1-\rho^{T}}{1-\rho}&=\frac{\eta_0}{2}\frac{1-(1-\gamma)^T}{\gamma}\\&\geq \frac{\eta_0}{2\gamma}\left(1-\left(1-\frac{\gamma T}{1+(T-1)\gamma}\right)\right)\\&\geq\frac{\eta_0}{2\gamma}\frac{\gamma T}{2}.
        \end{aligned}
    \end{equation}
    where we used the basic inequality $(1-x)^r\leq 1-\frac{rx}{1+(r-1)x}$ for $x\in [0,1], r>1$. Now, recalling $T=\Theta\left(\log\frac{1}{\alpha}/\eta_0\right)$ and $\gamma=\Theta\left(\eta_0/\log\frac{1}{\alpha}\right)$, we have $\frac{\eta_0}{2\gamma}\frac{\gamma T}{2}\geq \log (1/\alpha)$, which implies that
    after $T=\Theta(\log\frac{1}{\alpha}/\eta_0)$ iterations, the signal term satisfies $\|r_T\|\geq 1$. So, the only remaining part is to show that $\norm{r_T}=1\pm O(\delta\log\frac{r^{\prime}}{\delta})$. Recall that, based on the definition of $\bar{T}$, we have $\|r_{\bar{T}}\|<1$. Now, we assume that $\norm{r_{T-1}}< 1$, and $\norm{r_{T}}\geq 1$. Note that this assumption is without loss of generality, since $\bar{T}$ and $T$ have the same order. Then we have the following claim.
    \begin{claim}
    Either $1-\delta\log\frac{r'}{\delta} \lesssim \|r_{T-1}\|^2$, or $\|r_T\|\lesssim 1+\delta^2\log\frac{r'}{\delta}$.
    \end{claim} 
    \begin{proof}
      Assume that $\norm{\Delta_{T-1}}_F\geq 1 - \norm{r_{T-1}}^2\gtrsim \delta\log\frac{r'}{\delta}$. Then, by Proposition~\ref{signal-dynamics-noisy}, we have
       \begin{equation}
           \begin{aligned}
                \norm{r_{T}}-\norm{r_{T-1}}&\leq \frac{4}{3}\frac{\eta_0\rho^{T-1}(1-\norm{r_{T-1}}^2)}{\norm{\Delta_{T-1}}_F}\norm{r_{T-1}}+\frac{2\eta_0\rho^{T-1}\norm{E_{T-1}}^2}{\norm{\Delta_{T-1}}_F}\norm{r_{T-1}}+O(\delta\eta_0\rho^T)\\
                &\lesssim\frac{1}{\log\frac{r'}{\delta}} (1-\norm{r_{T-1}})+\delta^2\log\frac{r^{\prime}}{\delta}\\
                &\lesssim\frac{1}{\log\frac{r'}{\delta}} (\norm{r_T}-\norm{r_{T-1}})+\delta^2\log\frac{r^{\prime}}{\delta}
           \end{aligned}
       \end{equation}
       This implies that, for sufficiently small $\delta$, we have $\norm{r_{T}}-\norm{r_{T-1}}=O(\delta^2\log\frac{r^{\prime}}{\delta})$, thereby completing the proof.
    \end{proof}
    In summary, we showed that $1-\delta\log\frac{r'}{\delta} \lesssim \|r_{T-1}\|^2\leq 1$, or $1\leq \|r_T\|\lesssim 1+\delta^2\log\frac{r'}{\delta}$. On the other hand, we know that $\norm{E_t}\lesssim \delta\log\frac{r'}{\delta}$ for every $t\leq T$. This together with Lemma~\ref{decomposition-error-signal} completes the proof.$\hfill\square$

\section{Proof of Proposition \ref{uniform-convergence-noisy-l2}}
We divide our analysis into two cases. In the first case, we assume $p\sigma^2=\Omega(1)$. We have
\begin{equation}
    \begin{aligned}
        \sup_{X\in\mathbb{S}}\norm{Q(X)-X}_F&=\sup_{X,Y\in\mathbb{S}}\left|\frac{1}{m}\sum_{i=1}^{m}\inner{A_i}{X}\inner{A_i}{Y}+\frac{1}{m}\sum_{i\in S}s_i\inner{A_i}{Y}-\inner{X}{Y}\right|\\
         &\stackrel{(a)}{\geq} \sup_{Y\in\mathbb{S}}\left|\frac{1}{m}\sum_{i=1}^{m}\inner{A_i}{Y}^2+\frac{1}{m}\sum_{i\in S}s_i\inner{A_i}{Y}-1\right|\\
         & \geq \sup_{Y\in\mathbb{S}}\left|\frac{1}{m}\sum_{i\in S}s_i\inner{A_i}{Y}\right|-\sup_{Y\in\mathbb{S}}\left|\frac{1}{m}\sum_{i=1}^{m}\inner{A_i}{Y}^2-1\right|\\
        &\stackrel{(b)}{=}\norm{\frac{1}{m}\sum_{i\in S}s_iA_i}_F - \sup_{Y\in\mathbb{S}}\left|\frac{1}{m}\sum_{i=1}^{m}\inner{A_i}{Y}^2-1\right|.
    \end{aligned}
\end{equation}
where in (a) we add a constraint $X=Y$ to the supremum; and in (b) we use the Cauchy-Schwartz inequality and the variational form of the Frobenius norm. By the $\ell_2$-RIP for Gaussian measurements (Lemma \ref{l2-RIP}), we have
\begin{equation}
\label{eq70}
    \sup_{X,Y\in\mathbb{S}}\left|\frac{1}{m}\sum_{i=1}^{m}\inner{A_i}{X}\inner{A_i}{Y}+\frac{1}{m}\sum_{i\in S}s_i\inner{A_i}{Y}-\langle X, Y\rangle\right|\geq \norm{\frac{1}{m}\sum_{i\in S}s_iA_i}_F - \delta_1
\end{equation}
with probability of at least $1-Ce^{-cm\delta_1^2}$, given $m\gtrsim d^2$. The expectation and tail bound of $\norm{\frac{1}{m}\sum_{i\in S}s_iA_i}_F$ is provided in the following lemma.
\begin{lemma}
\label{concentration-lemma}
    For any $0<t<1$, we have
    \begin{equation}
       \mathbb{P}\left(\left|\norm{\frac{1}{m}\sum_{i\in S}s_iA_i}_F-\E\left[\norm{\frac{1}{m}\sum_{i\in S}s_iA_i}_F\right]\right|\geq t\right)\leq 2e^{-\frac{Cmt^2}{p\sigma^2d^2}},
     \end{equation}
    where $C$ is a universal constant. Moreover, the expectation is lower bounded as
    \begin{equation}
        \E\left[\norm{\frac{1}{m}\sum_{i\in S}s_iA_i}_F\right]\gtrsim \sqrt{\frac{p\sigma^2d^2}{m}}.
    \end{equation}
\end{lemma}
Before providing the proof of Lemma \ref{concentration-lemma}, we complete the proof of Proposition \ref{uniform-convergence-noisy-l2}. Based on the above lemma and \eqref{eq70}, we have
\begin{equation}
    \sup_{X,Y\in\mathbb{S}}\left|\frac{1}{m}\sum_{i=1}^{m}\inner{A_i}{X}\inner{A_i}{Y}+\frac{1}{m}\sum_{i\in S}s_i\inner{A_i}{Y}-1\right|\geq C\sqrt{\frac{p\sigma^2d^2}{m}} - \delta_1 - \delta_2,
\end{equation}
with probability of at least $1-Ce^{-c_1 m\delta_1^2}-e^{-c_2\frac{m\delta_2^2}{p\sigma^2d^2}}$. Hence, with the proper choice of $\delta_1,\delta_2$, we have
\begin{equation}
    \mathbb{P}\left(\sup_{X\in\mathbb{S}}\norm{Q(X)-X}_F\geq C\sqrt{\frac{p\sigma^2d^2}{m}}\right)\geq \frac{1}{2}.
\end{equation}
Since $p\sigma^2=\Omega(1)$, we can choose $C^{\prime}$ such that
\begin{equation}
    \mathbb{P}\left(\sup_{X\in\mathbb{S}}\norm{Q(X)-X}_F\geq C^{\prime}\sqrt{\frac{(1+p\sigma^2)d^2}{m}}\right)\geq \frac{1}{2}.
\end{equation}

In the second case, we assume that $p\sigma^2=O(1)$. Making a similar argument, we can show that there exists a universal constant $C$ such that
\begin{equation}
    \mathbb{P}\left(\sup_{X\in\mathbb{S}}\norm{Q(X)-X}_F\geq C\sqrt{\frac{d^2}{m}}\right)\geq \frac{1}{2}.
\end{equation}

Combining the two cases, the following inequality holds for an arbitrary $\sigma>0$
\begin{equation}
    \mathbb{P}\left(\sup_{X\in\mathbb{S}}\norm{M_2(X)-\bar{M}_2(X)}_F\geq C^{\prime}\sqrt{\frac{(1+p\sigma^2)d^2}{m}}\right)\geq \frac{1}{2}.
\end{equation}
Which completes the proof of Proposition~\ref{uniform-convergence-noisy-l2}.$\hfill\square$

Now, we present the proof for Lemma~\ref{concentration-lemma}.
\begin{proof}[Proof of Lemma \ref{concentration-lemma}]
   For simplicity, we denote $B=\frac{1}{m}\sum_{i\in S}s_iA_i$. First, we prove the lower bound on the expectation. Note that, conditioned on $s_i$, we have $B_{j,k}=\frac{1}{m}\sum_{i\in S}s_i A^i_{j,k}\sim N\left(0,\frac{1}{m^2}\sum_{i\in S}s_i^2\right)$. Then, by invoking Theorem 3.1.1. in \citep{vershynin2019high}, we have
   \begin{equation}
       \begin{aligned}
           \E\left[\norm{B}_F\right]&=\E\left[\E\left[\norm{B}_F\right]|s_i,i\in S\right]\\
           &\gtrsim \E\left[\frac{d}{m}\sqrt{\sum_{i\in S}s_i^2}\right]\\
           &\gtrsim \frac{\sigma d}{m}\sqrt{pm}=\sqrt{\frac{p\sigma^2d^2}{m}}.
       \end{aligned}
   \end{equation}
   
   Now, we show that $\norm{B}_F$ is a sub-exponential random variable. First, for arbitrary indices $i, j, k$, the random variable $s_iA^i_{j,k}$ is sub-exponential according to Lemma \ref{sub-gaussian-sub-exponential} since $\norm{s_iA^i_{j,k}}_{\psi_1}\leq \norm{s_i}_{\psi_{2}}\norm{A^i_{j,k}}_{\psi_{2}}=\Theta(\sigma)$. This implies that $\norm{B_{j,k}}_{\psi_1}=\Theta\left(\sqrt{\frac{p\sigma^2}{m}}\right)$. Finally, we have
   \begin{equation}
       \begin{aligned}
            \norm{\norm{B}_F}_{\ell^{2k}}&=\left(\norm{\sum_{j,k}B_{j,k}^2}_{\ell^k}\right)^{1/2}\\
            &\stackrel{(a)}{\leq}\left(\sum_{j,k}\norm{B^2_{j,k}}_{\ell^k}\right)^{1/2}\\
            &=d \norm{B_{j,k}}_{\ell^{2k}}\lesssim \sqrt{\frac{p\sigma^2d^2}{m}}k.
       \end{aligned}
   \end{equation}
   which implies that $\norm{B}_F$ is sub-exponential with sub-exponential norm $O\left(\sqrt{\frac{p\sigma^2d^2}{m}}\right)$ due to the equivalent definition of sub-exponential random variable (see Definition \ref{def-sub-exponential}). Note that in (a) we used the Minkowski inequality. Given the lower bound on the expected value, the tail bound directly follows from the tail of sub-exponential distribution.
\end{proof}

\section{Auxiliary Lemmas}
\subsection{Restricted Isometry Property}
\begin{lemma}
    \label{covering}
    Let $\mathbb{S}_r=\{X\in\R^{d\times d}:\rank(X)\leq r, \norm{X}_F=1\}$. Then, there exists an $\epsilon$-covering $\mathbb{S}_{\epsilon, r}$ with respect to the Frobenius norm satisfying $|\mathbb{S}_{\epsilon}|\leq \left(\frac{9}{\epsilon}\right)^{(2d+1)r}$. 
\end{lemma}


\begin{lemma}[$\ell_2$-RIP, Theorem 4.2 in \citep{recht2010guaranteed}]
\label{l2-RIP}
    Fix $0<\delta<1$, suppose that the measurement matrices $\{A_i\}_{i=1}^m$ have i.i.d. standard Gaussian entries. Then, we have
    \begin{equation}
        \sup_{X\in\mathbb{S}_r}\left|\frac{1}{m}\sum_{i=1}^{m}\inner{A_i}{X}^2-\norm{X}_F^2\right|\leq \delta.
    \end{equation}
    with probability of at least $1-Ce^{c_1dr\log\frac{1}{\delta}-c_2m\delta^2}$.
\end{lemma}

\subsection{Basic Probability}
\begin{lemma}[Conditional Gaussian Variable in Bivariate Case]
    For two Gaussian random variables $X\sim\N(\mu_1,\sigma_1^2),Y\sim\N(\mu_2,\sigma_2^2)$ with correlation coefficient $\rho$, we have
    \begin{equation}
        X|Y=a\sim \N\left(\mu_{1}+\frac{\sigma_{1}}{\sigma_{2}} \rho\left(a-\mu_{2}\right),\left(1-\rho^{2}\right) \sigma_{1}^{2}\right).
    \end{equation}
\end{lemma}

\begin{definition}[Sub-Gaussian random variable]
\label{def-sub-gaussian}
    We say a random variable $X\in \R$ with expectation $\E[X]=\mu$ is $\sigma^2$-sub-Gaussian if for all $\lambda\in \R$, we have $\E\left[e^{\lambda (X-\mu)}\right]\leq e^{\frac{\lambda^2\sigma^2}{2}}$. This definition is equivalent to the following statements
\begin{itemize}
    \item (Tail bound) For any $t>0$, we have $\mathbb{P}(|X-\mu|\geq t)\leq 2e^{-\frac{t^2}{2\sigma^2}}$.
    \item (Moment bound) For any positive integer $p$, we have $\norm{X}_{\ell^p}=\left(\E\left[|X|^p\right]\right)^{1/p}\lesssim \sigma \sqrt{p}$.
\end{itemize}
Moreover, the sub-Gaussian norm of $X$ is defined as $\norm{X}_{\psi_2}:=\sup_{p\geq 1} \left\{p^{-1/2}\norm{X}_{\ell^p}\right\}$.
\end{definition}

For sum of independent sub-Gaussian random variables, their sub-Gaussian norm can be bounded via the following lemma.

\begin{lemma}[Proposition 2.6.1 in \citep{vershynin2019high}]
\label{sum-of-ind-sub-gaussian}
    Let $X_1,\cdots,X_m$ be a series independent zero-mean sub-Gaussian variables, then $\sum_{i=1}^{m}$ is sub-Gaussian and
    \begin{equation}
        \norm{\sum_{i=1}^{m} X_i}_{\psi_2}^2\lesssim \sum_{i=1}^{m} \norm{X_i}_{\psi_2}^2.
    \end{equation}
\end{lemma}

\begin{definition}[Sub-exponential random variable]
\label{def-sub-exponential}
    A random variable $X$ with expectation $\mu$ is sub-exponential if there exists $(\mu,\alpha)$, such that $\E\left[e^{\lambda(X-\mu)}\right]\leq e^{\frac{\lambda^2\nu^2}{2}}$ for all $|\lambda|\leq \alpha$. This definition is equivalent to the following statements:
\begin{itemize}
    \item (tail bound) There exists a universal constant $C$, for any $t>0$, we have $\mathbb{P}(|X-\mu|\geq t)\leq 2e^{-Ct}$.
    \item (moment bound) For any positive integer $p$, we have $\norm{X}_{\ell^p}=\left(\E\left[|X|^p\right]\right)^{1/p}\lesssim p$.
\end{itemize}
Moreover, the sub-exponential norm of $X$ is defined as $\norm{X}_{\psi_1}:=\sup_{p\geq 1} \left\{p^{-1}\norm{X}_{L^p}\right\}$.
\end{definition}

For sub-Gaussian and sub-exponential random variables, we have the following lemma to illustrate their relations.

\begin{lemma}
\label{sub-gaussian-sub-exponential}
The following statements hold
    \begin{itemize}
        \item (Lemma 2.7.6 in \citep{vershynin2019high}) A random variable $X$ is sub-Gaussian if and only if $X^2$ is sub-exponential. Moreover, $\norm{X}_{\psi_2}^2=\norm{X^2}_{\psi_1}$.
        
        \item (Lemma 2.7.7 in \citep{vershynin2019high}) Let $X$ and $Y$ be sub-Gaussian random variables. Then $XY$ is sub-exponential. Moreover, $\norm{XY}_{\psi_1}\leq \norm{X}_{\psi_2}\norm{Y}_{\psi_2}$.
    \end{itemize}
\end{lemma}

\subsection{Basic Inequalities}
\begin{lemma}[Bernoulli inequality]
    \label{bernoulli-inequality}
    The following inequality holds
    \begin{equation}
        (1+x)^r\leq 1 + \frac{rx}{1-(r-1)x}, \quad \text{for } x\in \left[-1, \frac{1}{r-1}\right), r\geq 1.
    \end{equation}
\end{lemma}

\end{document}